\documentclass[a4paper,11pt]{article}
\usepackage[top=4cm, bottom=4cm, left=2.8cm, right=2.8cm]{geometry}

\usepackage[colorlinks = true,
           linkcolor = red,
           urlcolor  = blue,
           citecolor = blue,
           anchorcolor = blue]{hyperref}

\usepackage{microtype}
\usepackage{graphicx}
\usepackage{subfigure}
\usepackage{booktabs} 

\usepackage{mathrsfs}  
\usepackage{enumitem}

\usepackage{natbib}

\usepackage{amsfonts}       
\usepackage{nicefrac}       
\usepackage{amssymb}
\usepackage{amsmath}
\usepackage{amsthm}
\usepackage{lmodern}
\usepackage{thmtools, thm-restate}
\usepackage[nameinlink,capitalize]{cleveref}

\usepackage{algorithm,algorithmic}
\newcommand{\tr}{\textnormal{tr}}
\newcommand{\xx}{\mathcal{X}}
\newcommand{\hh}{\mathcal{H}}

\newcommand{\EE}{\mathbb{E}}

\newcommand{\hs}{\mathsf{HS}}

\newcommand{\skx}{\mathsf{m}_{\subx}}

\newcommand{\skyt}{\mathsf{m}_{\suby,t}}
\newcommand{\sky}{\mathsf{m}_{\suby}}

\newcommand{\yy}{\mathcal{Y}}
\newcommand{\GG}{\mathcal{G}}
\newcommand{\ee}{\mathcal{E}}
\newcommand{\rr}{\mathcal{R}}

\renewcommand{\tt}{\mathcal{T}}
\newcommand{\conset}{\mathcal{C}}
\newcommand{\risk}{\mathcal{R}}
\newcommand{\hy}{{\mathcal{H}_{\suby}} }
\newcommand{\hx}{{\mathcal{H}_{\subx}} }
\newcommand{\R}{\mathbb{R}}
\def\nor #1{\Vert #1 \Vert}
\def\tnorm #1{\Vert #1 \Vert_*}
\def\opnorm #1{\Vert #1 \Vert_{\rm op}}
\def\abs #1{\vert #1 \vert}
\def\sp #1#2{\langle #1,#2 \rangle}

\newcommand{\bigx}[1]{\ensuremath{\mathop{}\mathopen{}\mathcal{O}\mathopen{}}}

\newcommand\smallx[1]{
    \mathchoice
    {
      \scriptstyle\mathcal{X}
    }
    {
      \scriptstyle\mathcal{X}
    }
    {
      \scriptscriptstyle\mathcal{X}
    }
    {
      \scalebox{0.8}{$\scriptscriptstyle\mathcal{O}$}
    }
}
\newcommand{\subx}{\scalebox{0.5}{$\mathcal{X}$}}
\newcommand{\suby}{\scalebox{0.5}{$\mathcal{Y}$}}
\newcommand{\msf}[1]{\mathsf{#1}}

\newcommand{\fstar}{{f_*}}
\newcommand{\gstar}{{g_*}}
\newcommand{\fhat}{\hat{f}}
\newcommand{\ghat}{\hat{g}}

\newcommand{\loss}{\ell}
\newcommand{\surrloss}{{\msf{L}}}

\newcommand{\closs}{\msf{q}_\loss}

\newcommand{\encoding}{{\msf{c}}}
\newcommand{\decoding}{{\msf{d}}}

\providecommand{\scal}[2]{\left\langle{#1},{#2}\right\rangle}

\providecommand{\nor}[1]{\left\|{#1}\right\|}

\declaretheorem[name=Theorem,refname=Thm.]{theorem}
\declaretheorem[name=Definition,refname=Def.]{definition}
\declaretheorem[name=Lemma,sibling=theorem]{lemma}

\declaretheorem[name=Proposition,refname=Prop.,sibling=theorem]{proposition}

\declaretheorem[name=Assumption,refname=Asm.]{assumption}

\declaretheorem[name=Example,refname=Ex.]{example}

\newcommand{\eqals}[1]{\begin{align*}#1\end{align*}}
\newcommand{\eqal}[1]{\begin{align}#1\end{align}}
\newcommand{\argmin}{\operatornamewithlimits{argmin}}
\newcommand{\argmax}{\operatornamewithlimits{argmax}}

\renewcommand{\paragraph}[1]{~\newline\noindent{\bf #1.}}

\crefname{assumption}{Asm.}{Assumptions}
\crefname{equation}{Eq.}{Eqs.}
\crefname{figure}{Fig.}{Figs.}
\crefname{section}{Sec.}{Secs.}
\crefname{algorithm}{Alg.}{Algs.}

\usepackage{etoolbox}
\AfterEndEnvironment{restatable}{\noindent\ignorespaces}
\AfterEndEnvironment{theorem}{\noindent\ignorespaces}
\AfterEndEnvironment{definition}{\noindent\ignorespaces}
\AfterEndEnvironment{example}{\noindent\ignorespaces}
\AfterEndEnvironment{remark}{\noindent\ignorespaces}
\AfterEndEnvironment{assumption}{\noindent\ignorespaces}

\title{\sffamily\LARGE\bf Leveraging Low-Rank Relations Between Surrogate Tasks in Structured Prediction }
\author{\small Giulia Luise$^{1}$ \\ {\scriptsize\em g.luise.16@ucl.ac.uk} \and \small Dimitris Stamos$^1$ \\ {\scriptsize\em d.stamos@cs.ucl.ac.uk} \\ \and \small Massimiliano Pontil$^{1,2}$ \\ {\scriptsize\em  m.pontil@cs.ucl.ac.uk ~~} \and \small Carlo Ciliberto$^{3}$ \\ {\scriptsize\em c.ciliberto@imperial.ac.uk} \\ $ $ \\  }

\begin{document}

\maketitle

\begin{abstract}
\noindent We study the interplay between surrogate methods for structured prediction and techniques from multitask learning designed to leverage relationships between surrogate outputs.  
We propose an efficient algorithm based on trace norm regularization which, differently from previous methods, does not require explicit knowledge of the coding/decoding functions of the surrogate framework. 
As a result, our algorithm can be applied to the broad class of problems in which the surrogate space is large or even infinite dimensional. We study excess risk bounds for trace norm regularized structured prediction, implying the consistency and learning rates for our estimator. We also identify relevant regimes in which our approach can enjoy better generalization performance than previous methods. 
Numerical experiments on ranking problems indicate that enforcing low-rank relations among surrogate outputs may indeed provide a significant advantage in practice.
\end{abstract}


\section{Introduction}

\footnotetext[1]{Department of Computer Science, University College London, WC1E 6BT London, United Kingdom}\footnotetext[2]{Computational Statistics and Machine Learning, Istituto Italiano di Tecnologia, 16100 Genova, Italy}\footnotetext[2]{Department of Electrical and
Electronic Engineering, Imperial College London, SW7 2BT, United Kingdom}
The problem of structured prediction is receiving increasing attention in machine learning, due to its wide practical importance \citep{bakir2007predicting,nowozin2011structured} and 
the theoretical challenges in designing principled learning procedures \citep{taskar2004max,taskar2005learning,london2016stability,cortes2016structured}. A key aspect of this problem is
the non-vectorial nature of the output space, e.g. graphs, permutations, and manifolds. Consequently, traditional regression and classification algorithms are not well-suited to these settings and more sophisticated methods need to be developed.

%
Among the most well-established strategies for structured prediction are the so-called {\em surrogate methods} \citep{bartlett2006}. Within this framework, a coding function is designed to embed the structured output into a linear space, where the resulting problem is solved via standard supervised learning methods.
Then, the solution of the surrogate problem is pulled back to the original output space 
by means of 
a decoding procedure, which allows one to recover the structured prediction estimator under suitable assumptions. In most cases, the surrogate learning problem amounts to a vector-valued regression in a possibly infinite dimensional space. The prototypical choice for such surrogate estimator is given by regularized least squares in a reproducing kernel Hilbert space, as originally considered in \citep{KDE,Cortes:2005,bartlett2006} and then explored in \citep{Multiclasssimplexcoding,kadri13,brouard2016input,CilibertoRR16,osokin2017structured}. 




The principal goal of this paper is to extend the surrogate approaches to methods that encourage structure among the outputs. Indeed, a large body of work from traditional multitask learning has shown that leveraging the relations among multiple outputs may often lead to better estimators
\citep[see e.g.][and references therein]{Maurer:2006,Caponnetto2007,MicchelliMP13}. However, previous methods that propose to apply multitask strategies to surrogate frameworks \citep[see e.g.][]{alvarez2012kernels,fergus2010semantic} heavily rely on the explicit knowledge of the encoding function. As a consequence they are not applicable when the surrogate space is large or even infinite dimensional.



%

\paragraph{Contributions} We propose a new algorithm based on low-rank regularization for structured prediction that builds upon the surrogate framework in \citep{CilibertoRR16,ciliberto2017consistent}. Differently from previous methods, our algorithm does not require explicit knowledge of the encoding function. In particular, by leveraging approaches based on the variational formulation of trace norm regularization \citep{srebro2005}, we are able to derive an efficient learning algorithm also in the case of infinite dimensional surrogate spaces. 


We characterize the generalization properties of the proposed estimator by proving excess risk bounds for the corresponding least-squares surrogate estimator that extend previous results \citep{Bach08}. In particular, in line with previous work on the topic \citep{PontilM13}, we identify settings in which the trace norm regularizer can provide significant advantages over standard $\ell_2$ regularization. While similar findings have been obtained in the case of a Lipschitz loss, to our knowledge this is a novel result for least-squares regression with trace norm regularization. In this sense, the implications of our analysis extend beyond structured prediction and apply to settings such as collaborative filtering with side information \citep{Abernethy2008}. We evaluate our approach on a number of learning-to-rank problems. In our experiments the proposed method significantly outperforms all competitors, suggesting that encouraging the surrogate outputs to span a low-rank space can be beneficial also in structured prediction settings. 

\paragraph{Paper Organization} \cref{sec:background} reviews surrogate methods and the specific framework adopted in this work. \cref{sec:trace norm} introduces the proposed approach to trace norm regularization and proves that it does not leverage explicit knowledge of coding and surrogate space. \cref{sec:stat_analysis} describes the statistical analysis of the proposed estimator both in a vector-valued and multi-task learning setting. \cref{sec:exps} reports on experiments and \cref{sec:conclusions} discusses future research directions.

\section{Background}\label{sec:background}

Our proposed estimator belongs to the family of surrogate methods \citep{bartlett2006}. This section reviews the main ideas behind these approaches. 

\subsection{Surrogate Methods}

Surrogate methods are general strategies to address supervised learning problems. Their goal is to learn a function $f:\xx\to\yy$ minimizing the {\em expected risk} of a distribution $\rho$ on $\xx\times\yy$
\eqal{
\label{eq:struc_pb} \ee(f):=\int_{\xx\times\yy} \loss(f(x),y)\,d\rho(x,y),
}
given only $n$ observations $(x_i,y_i)_{i=1}^n$ independently drawn from $\rho$, which is unknown in practice. Here $\loss:\yy\times\yy\to\R$ is a loss measuring prediction errors.

Surrogate methods have been conceived to deal with so-called {\em structured prediction} settings, namely supervised problems where $\yy$ is not a vector space but rather a ``structured'' set (of e.g. strings, graphs, permutations, points on a manifold, etc.). Surrogate methods have been successfully applied to problems such as classification \citep{bartlett2006}, multi-labeling \citep{gao2013,Multiclasssimplexcoding} or ranking \citep{duchi2010consistency}. They follow an alternative route to standard empirical risk minimization (ERM), which instead consists in directly finding the model that best explains training data within a prescribed hypotheses space.
%

~\newline\noindent Surrogate methods are characterized by three phases:
%
\begin{enumerate}
\item {\bf Coding.} Define an embedding $\encoding:\yy\to\hh$, where $\hh$ is a Hilbert space. Map $(x_i,y_i)_{i=1}^n$ to a ``surrogate'' dataset $(x_i,\encoding(y_i))_{i=1}^n$. 
\item {\bf Learning.} Define a surrogate loss $\surrloss:\hh\times\hh\to\R$. Learn a surrogate estimator $\ghat:\xx\to\hh$ via ERM on $(x_i,\encoding(y_i))_{i=1}^n$.
\item {\bf Decoding.} Define a decoding $\decoding:\hh\to\yy$ and return the structured prediction estimator $\fhat = d \circ\ghat:\xx\to\yy$.
\end{enumerate}

\noindent Below we give an well-known example of surrogate framework often used in multi-class classification settings.

\begin{example}[One Vs All]\label{ex:ova-classification} $\yy = \{1,\dots,T\}$ is a set of $T$ classes and $\loss$ is the $0$-$1$ loss. Then: $1)$ The coding is $\encoding:\yy\to\hh=\R^T$ with $\encoding(i) = e_i$, the vector of all $0$s but $1$ at the $i$-th entry. $2)$ $\ghat:\xx\to\R^T$ is learned by minimizing a surrogate loss $\surrloss:\R^T\times\R^T\to\R$ (e.g. least-squares). $3)$ The classifier is $\fhat(x) = \decoding(\ghat(x))$, with decoding $\decoding(v) = \argmax_{i=1}^T\{v_i\}$ for any $v\in\R^T$.
\end{example}
\noindent A key element of surrogate methods is the choice of the loss $\surrloss$. Indeed, since $\hh$ is linear (e.g. $\hh = \R^T$ in \Cref{ex:ova-classification}), if $\surrloss$ is convex it is possible to learn $\ghat$ efficiently by means of standard ERM. However, this opens the question of characterizing how the surrogate risk
\eqal{\label{eq:surrogate-risk}
	\rr(g) = \int \surrloss(g(x),\encoding(y))~d\rho(x,y)
}
is related to the original risk $\ee(f)$. In particular let $\fstar:\xx\to\yy$ and $\gstar:\xx\to\hh$ denote the minimizers of respectively $\ee(f)$ and $\rr(g)$. We require the two following conditions:

\begin{itemize}
\item {\bf Fisher Consistency.} $\ee(\decoding\circ\gstar) = \ee(\fstar)$.\\
\item {\bf Comparison Inequality.} For any $g:\xx\to\hh$, there exists a continuous nondecreasing function $\sigma:\R\to\R_+$, such that $\sigma(0)=0$ 
	\eqal{\label{eq:comparison-general}
		\ee(\decoding\circ g) - \ee(\fstar) \leq \sigma(\rr(g) - \rr(\gstar)).
}  
\end{itemize}
\noindent Fisher consistency guarantees the coding/decoding framework to be coherent with the original problem. The comparison inequality suggests to focus the theoretical analysis on $\ghat$, since learning rates for $\ghat$ directly lead to learning rates for $\fhat = \decoding\circ\ghat$.


\subsection{SELF Framework}

A limiting aspect of surrogate methods is that they are often tailored around individual problems. An exception 
is the framework in \citep{CilibertoRR16}, which provides a general strategy to identify coding, decoding and surrogate space for a variety of learning problems. The key condition in this settings is for the loss $\loss$ to be SELF:
\begin{definition}[SELF]\label{def:self}
A function $\ell:\yy\times \yy\rightarrow \R$ is a {\em Structure Encoding Loss Function (SELF)} if there exist a separable Hilbert space $\hy$, a continuous map $\psi:\yy\to \hy$ and $V:\hy\to\hy$ a bounded linear operator, such that for all $y,y'\in\yy$
\eqal{\label{eq:self}
	 \ell(y,y')=\sp{\psi(y)}{V \psi(y')}_{\hy}.
} 
\end{definition}  
The condition above is quite technical, but it turns out to be very general: it was shown in~\citep{CilibertoRR16,ciliberto2017consistent} that most loss functions used in machine learning in settings such as regression, robust estimation, classification, ranking, etc., are SELF. 

We can design surrogate frameworks ``around'' a SELF $\loss$, by choosing (\textit{Coding}) the map $\encoding = \psi:\yy\to\hy$, the least-squares (\textit{Surrogate loss}) $\surrloss(h,h') = \|h - h'\|^2_\hy$ and (\textit{Decoding}) $\decoding:\hy\to\yy$, defined for any $h\in\hy$ as 
\eqal{\label{eq:decoding}
	\decoding(h) = \textstyle{ \argmin_{y\in\yy} } ~ \scal{\psi(y)}{Vh}_\hy.
}
The resulting is a sound surrogate framework as summarized by the theorem below.

\begin{theorem}[Thm. 2 in \citep{CilibertoRR16}]\label{thm:fisher}
Let $\loss$ be SELF and $\yy$ a compact set. Then, the SELF framework introduced above is Fisher consistent. Moreover, it satisfies the comparison inequality \eqref{eq:comparison-general} with $\sigma(\cdot) = \closs\sqrt{\cdot}$, where $\closs = \|V\|\sup_{y\in\yy}\|\psi(y)\|_\hy$. 
\end{theorem}
{\bf Loss trick.} A key aspect of the SELF framework is that, in practice, the resulting algorithm {\em does not require explicit knowledge of the coding/decoding and surrogate space} (only needed for the theoretical analysis). To see this, let $\xx = \R^d$ and $\hy=\R^T$ and consider the parametrization $g(x) = G x$ of functions $g:\xx\to\hy$, with $G\in\R^{T \times d}$ a matrix. We can perform Tikhonov regularization to learn the matrix $\hat{G}$ minimizing the (surrogate) empirical risk
\eqal{\label{eq:ghat}
	\min_{G\in\R^{d \times T}} ~\frac{1}{n}\sum_{i=1}^n \nor{G x_i - \psi(y_i)}_{\hy}^2 + \lambda \nor{G}_{\hs}^2,
}
where $\|G\|_\hs$ is a Hilbert-Schmidt (HS) (or Frobenius) norm regularizer and $\lambda>0$. A direct computation gives 
a closed form expression for $\ghat:\xx\to\hy$, namely
\eqal{
	& \ghat(x) = \hat{G} x = \sum_{i=1}^n \alpha_i(x) \psi(y_i),\quad \textrm{with}  \label{eq:ghat-linear-combination} \\
	& \alpha(x) = (\alpha_1(x),\dots,\alpha_n(x))^\top = (K_{\subx} + n\lambda I)^{-1} v_x,\label{eq:alphas}
}
for every $x\in\xx$ \citep[see e.g.][]{alvarez2012kernels}. Here $K_{\subx}\in\R^{n \times n}$ is the empirical kernel matrix of the linear kernel $k_{\subx}(x,x') = x^\top x'$ and $v_x\in\R^{n}$ is the vector with $i$-th entry $(v_x)_i = k_{\subx}(x,x_i)$. 

Applying the SELF decoding in \cref{eq:decoding} to $\ghat$, we have for all $x\in\xx$
\eqal{\label{eq:estimator}
	\fhat(x) = \decoding(\ghat(x)) = {\argmin_{y\in\yy}}~\sum_{i=1}^n \alpha_i(x)\loss(y,y_i).
}
This follows by combining the SELF property $\loss(y,y_i) = \scal{\psi(y)}{V~\psi(y_i)}_\hy$ with $\ghat$ in \Cref{eq:ghat-linear-combination} and the linearity of the inner product. \cref{eq:estimator} was originally dubbed ``loss trick'' since it avoids explicit knowledge of the coding $\psi$, similarly to the feature map for the kernel trick \citep{scholkopf2002learning}. 

The characterization of $\fhat$ in terms of an optimization problem over $\yy$ (like in \cref{eq:estimator}) is a common practice to most structured prediction algorithms. In the literature, such decoding process is referred to as the inference \citep{nowozin2011structured} or pre-image \citep{brouard2016input,Cortes:2005,KDE} problem. We refer to \citep{honeine2011preimage,bakir2007predicting,nowozin2011structured} for examples on how these problems are addressed in practice. 


\paragraph{General Setting} The derivation above holds also when $\hy$ is infinite dimensional and when using a positive definite kernel $k_{\subx}:\xx\times\xx\to\R$ on $\xx$. Let $\hx$ be the reproducing kernel Hilbert space (RKHS) induced by $k_{\subx}$ and $\phi:\xx\to\hx$ a corresponding feature map \citep{aronszajn1950theory}. We can parametrize $g:\xx\to\hy$ as $g(\cdot) = G\phi(\cdot)$, with $G\in\hy\otimes\hx$ the space of Hilbert-Schmidt operators from $\hx$ to $\hy$ (the natural generalization of $\R^{d \times T} = \R^d \otimes\R^T$ to the infinite setting). The problem in \Cref{eq:ghat} can still be solved in closed form analogously to \cref{eq:ghat-linear-combination}, with now $K_{\subx}$ the empirical kernel matrix of $k$ \citep{Caponnetto2007}. This leads to the decoding for $\fhat$ as in \cref{eq:estimator}.


\section{Low-Rank SELF Learning}\label{sec:trace norm}

Building upon the SELF framework, we discuss the use of multitask regularizers to exploit potential relations among the surrogate outputs. Our analysis is motivated by observing that \Cref{eq:ghat} is equivalent to learning multiple (possibly infinitely many) scalar-valued functions
\eqal{\label{eq:independent-problems}
	\min_{\{g_t\}\in\hx} \frac{1}{n}\sum_{t\in\mathcal{T}}\sum_{i=1}^n(g_t(x_i) - \varphi_t(y_i))^2 + \lambda \|g_t\|_\hx^2,
}
where, given a basis $\{e_t\}_{t\in\tt}$ of $\hy$ with $\tt\subseteq\mathbb{N}$, we have denoted $\psi_t(y) = \scal{e_t}{\psi(y)}_\hy$ for any $y\in\yy$ and $t\in\tt$ (for instance, in the case of \cref{eq:ghat} we have $t\in\{1,\dots,T\}$). Indeed, from the literature on vector-valued learning in RKHS~\citep[see e.g][]{micchelli2005learning}, we have that for $g:\xx\to\hy$ parametrized by an operator $G\in\hy\otimes\hx$, any $g_t:\xx\to\R$ defined by $g_t(\cdot) = \scal{e_t}{g(\cdot)}_\hy$, is a function in the RKHS $\hx$ and, moreover, $\|G\|_\hs^2= \sum_{t\in\tt} \|g_t\|_\hx^2$.  

The observation above implies that we are learning the surrogate ``components'' $g_t$ as separate problems or {\em tasks}, an approach often referred to as ``independent task learning'' within the multitask learning (MTL) literature \citep[see e.g.][]{micchelli2005learning,evgeniou2005learning,Argyriou2008}. In this respect, a more appropriate strategy would be to leverage potential relationships between such components during learning. In particular, we consider the problem
\eqal{\label{eq:trace norm-problem}
	\min_{G\in\hy\otimes\hx}\frac{1}{n}\sum_{i=1}^n \nor{G\phi(x_i) - \psi(y_i)}_{\hy}^2 + \lambda \tnorm{G},
}
where $\tnorm{G}$ denotes the trace norm, namely the sum of the singular values of $G$. Similarly to the $\ell_1$-norm on vectors, the trace norm favours sparse (and thus low-rank) solutions. Intuitively, encouraging $G$ to be low-rank reduces the degrees of freedom allowed to the individual tasks $g_t$. This approach has been extensively investigated and successfully applied to several MTL settings, \citep[see e.g.][]{Argyriou2008,Bach08,Abernethy2008,PontilM13}.
%

In general, the idea of combining MTL methods with surrogate frameworks has already been studied in settings such as classification or multi-labeling 
\citep[see e.g.][]{alvarez2012kernels,fergus2010semantic}. However, these approaches require to explicitly use the coding/decoding and surrogate space within the learning algorithm. This is clearly unfeasible when $\hy$ is large or infinite dimensional.

\paragraph{SELF and Trace Norm MTL} In this work we leverage the SELF property outlined in \cref{sec:background} to derive an algorithm that overcomes the issues above and {\em does not} require explicit knowledge of the coding map $\psi$. However, our approach still requires to access the matrix $K_{\suby}\in\R^{n \times n}$ of inner products $(K_{\suby})_{ij} = \scal{\psi(y_i)}{\psi(y_j)}_\hy$ between the training outputs. When the surrogate space $\hy$ is a RKHS, $K_{\suby}$ corresponds to an empirical {\em output kernel matrix}, which can be efficiently computed. This motivates us to introduce the the following assumption.
\begin{assumption}[SELF \& RKHS]\label{asm:self-rkhs} The loss $\loss:\yy\times\yy\to\R$ is SELF with $\hy$ a RKHS on $\yy$ with reproducing kernel $k_{\suby}(y,y') = \scal{\psi(y)}{\psi(y')}_\hy$ for any $y,y'\in\yy$. 
\end{assumption} 
The assumption above imposes an additional constraint on $\loss$ and thus on the applicability of \cref{alg:trace norm}. However, it was shown in \citep{CilibertoRR16} that this requirement is always satisfied by any loss when $\yy$ is a discrete set. In this case the output kernel is the $0$-$1$ kernel, that is, $k_{\suby}(y,y') = \delta_{y=y'}$. Moreover, it was recently shown that \Cref{asm:self-rkhs} holds for any smooth $\loss$ on a compact set $\yy$ by choosing $k_{\suby}(y,y') = \exp(-\|y - y\|/\sigma)$, the Abel kernel with hyperparameter $\sigma>0$ \citep{luise2018differential}.

\paragraph{Algorithm} Standard methods to solve \cref{eq:trace norm-problem}, such as forward-backward splitting, require to perform the singular value decomposition of the estimator at every iteration \citep{mazumder2010spectral}. This is prohibitive for large scale applications and, to overcome these drawbacks, algorithms exploiting the variational form of the trace norm
\eqals{
\nor{G}_* = \frac{1}{2}\inf\Big\{ \|A\|^2_{\hs} + \nor{B}^2_{\hs}~:~ G = AB^*,~r\in\mathbb{N},
~A\in \hy \otimes \R^r,~~ B\in \hx\otimes\R^r \Big\},
}
have been considered \citep[see e.g.][]{srebro2005} (here $B^*$ denotes the adjoint of $B$). Using this characterization, \cref{eq:trace norm-problem} is reformulated as the problem of minimizing
\eqal{\label{eq:factorized-problem}
	\hspace{-.27truecm}\frac{1}{n}\sum_{i=1}^n \nor{AB^*\phi(x_i) {-} \psi(y_i)}_{\hy}^2 {+} \lambda \big(\|A\|_\hs^2 {+} \|B\|_\hs^2\big),
}
over the operators $A\in \hy \otimes \R^r$ and $B\in \hx\otimes\R^r$, where $r\in\mathbb{N}$ is now a further hyperparameter. 
The functional in \cref{eq:factorized-problem} is smooth and methods such as gradient descent can be  applied. Interestingly, despite the functional being non-convex, guarantees on the global convergence in these settings have been explored \citep{journee2010low}. 

In the SELF setting, minimizing \cref{eq:factorized-problem} has the additional advantage that it allows us to derive an analogous of the loss trick introduced in \cref{eq:estimator}. In particular, the following result shows how each iterate of gradient descent can be efficiently ``decoded'' into a structured prediction estimator according to \cref{alg:trace norm}. 
\begin{algorithm}[t]
\newcommand{\tn}{\textnormal}
\caption{{\sc Low-Rank SELF Learning}}\label{alg:trace norm}
\begin{algorithmic}
  \vspace{0.25em}
  \STATE {\bfseries Input:} $K_{\subx}, K_{\suby}\in\R^{n \times n}$ empirical kernel matrices for input and output data, $\lambda$ regularizer, $r$ rank, $\nu$ step size, $k$ number of iterations.  
  
  \vspace{0.45em}
  \STATE {\bfseries Initalize:} Sample $M_0,N_0\in\R^{n \times r}$ randomly.  
  
  \vspace{0.45em}
  \STATE {\bfseries For} $j=1,\dots,k$:
      \STATE \quad~~ $M_{k+1}=(1-\lambda\nu)M_k -\nu(K_{\subx} M_k N_k-I)K_{\suby} N_k $ ~~~
      \STATE \quad~~ $N_{k+1}=(1-\lambda\nu)N_k -\nu(N_k M_k^\top K_{\subx} -I)K_{\subx} M_k$~~~~
      
  \STATE~
    \STATE {\bfseries Return:} The weighting function $\alpha^{\mathsf{tn}}:\xx\to\R^n$ given, for any $x\in\xx$, by 
    $\alpha^{\mathsf{tn}}(x) = N_k M_k^\top v_x$\\
    \STATE ~~~~~~~~~~~~ where $v_x= (k_{\subx}(x,x_i))_{i=1}^n$
\end{algorithmic}
\end{algorithm}
\begin{restatable}[Loss Trick for Low-Rank SELF Learning]{theorem}{losstrick}\label{thm:loss-trick} Under \cref{asm:self-rkhs}, let $M,N\in\R^{n \times r}$ and $(A_k,B_k)$ be the $k$-th iterate of gradient descent on \cref{eq:factorized-problem} from $A_0 = \sum_{i=1}^n \phi(x_i)\otimes M^i$ and $B_0 = \sum_{i=1}^n \psi(y_i)\otimes N^i$, with $M^i,N^i$ denoting the $i$-th rows of $M$ and $N$ respectively. Let $\hat g_k:\xx\to\hy$ be such that $\hat g_k(\cdot) = A_kB_k^*\phi(\cdot)$. Then, the structured prediction estimator $\hat f_k = \decoding \circ \hat g_k:\xx\to\yy$ with decoding $\decoding$ in \cref{eq:decoding} is such that
\eqals{
    \hat f_k(x) = \argmin_{y\in\yy} \sum_{i=1}^n \alpha^{\mathsf{tn}}_i(x)~\loss(y,y_i)
}
for any $x\in\xx$, with $\alpha^{\mathsf{tn}}(x)\in\R^n$ the output of \cref{alg:trace norm} after $k$ iterations starting from $(M_0,N_0) = (M,N)$.
\end{restatable}
The result above shows that \cref{alg:trace norm} offers a concrete algorithm to perform the SELF decoding $\hat f_k = \decoding \circ \hat g_k$ of the surrogate function $\hat g_k(\cdot) = A_kB_k^*\phi(\cdot)$ obtained after $k$ iterations of gradient descent on \cref{eq:factorized-problem}. Note that when $\hy$ is infinite dimensional it would be otherwise impossible to perform gradient descent in practice. In this sense, \cref{thm:loss-trick} can be interpreted as a representer theorem with respect to both inputs and outputs. The details of the proof are reported in \cref{sec:loss_tricks}; the key aspect is to show that every iterate $(A_j,B_j)$ of gradient descent on \cref{eq:factorized-problem} is of the form $A_j = \sum_{i=1}^n \phi(x_i)\otimes M^i$ and $B_j = \sum_{i=1}^n \psi(y_i)\otimes N^i$ for some matrices $M,N\in\R^{n\times r}$. Hence, the products $A_j^*A_j = M^\top K_{\subx} M$ and $B_j^*B_j = N^\top K_{\suby} N$ -- used in the optimization -- are $r\times r$ matrices that can be efficiently computed in practice, leading to \cref{alg:trace norm}.

We conclude this section by noting that, in contrast to trace norm regularization, not every MTL regularizer fits naturally within the SELF framework. 


\paragraph{SELF and other MTL Regularizer} A well-established family of MTL methods consists in replacing the trace norm $\tnorm{G}$ with $\tr(G A G^*)$ in \cref{eq:trace norm-problem}, 
 where $A\in\hy\otimes\hy$ is a positive definite linear operator enforcing specific relations on the tasks via a deformation of the metric of $\hy$ \citep[see][and references therein]{micchelli2005learning,Jacob:2008,alvarez2012kernels}. 
While in principle appealing also in surrogate settings, these approaches present critical computational and modelling challenges for the SELF framework: the change of metric induced by $A$ has a disruptive effect on the loss trick. As a consequence, an equivalent of \cref{thm:loss-trick} does not hold in general (see \cref{sec:C0} for a detailed discussion).

\section{Theoretical Analysis}\label{sec:stat_analysis}

In this section we study the generalization properties of low-rank SELF learning. Our analysis is indirect since we characterize the learning rates of the {\em Ivanov} estimator (in contrast to {\em Tikhonov}, see \cref{eq:trace norm-problem}), given by
\eqal{\label{eq:ivanov_tn}
	\hat G = \argmin_{\tnorm{G}\leq\gamma} ~\frac{1}{n}\sum_{i=1}^n \|G\phi(x_i)-\psi(y_i)\|_\hy^2.
} 
Indeed, while Tikhonov regularization is typically more convenient from a computational perspective, 
Ivanov regularization if often more amenable to theoretical analysis since it is naturally related to standard complexity measures for hypotheses spaces, such as Rademacher complexity, Covering Numbers or VC dimension \citep{Shalev-Shwart}. However, the two regularization strategies are {\em equivalent} in the following sense: for any $\gamma$ there exists $\lambda(\gamma)$ such that the minimizer of \cref{eq:trace norm-problem} (Tikhonov) is also a minimizer for \cref{eq:ivanov_tn} (Ivanov) with constraint $\gamma$ (and vice-versa). This follows from a standard Lagrangian duality argument leveraging the convexity of the two problems \citep[see e.g.][or \cref{sec:relation-between-tikhonov-ivanov} for more details]{Oneto:2016}. Hence, while our results in the following are reported for the Ivanov estimator from \Cref{eq:ivanov_tn}, they apply equivalently to Tikhonov in \cref{eq:trace norm-problem}.

We now proceed to present the main result of this section,  proving excess risk bounds for the trace norm surrogate estimator. In the following we assume a reproducing kernel $k_{\subx}$ on $\xx$ and $k_{\suby}$ on $\yy$ (according to \Cref{asm:self-rkhs}) and denote $\skx^2 = \sup_{x\in\xx} k_{\subx}(x,x)$ and $\sky^2 = \sup_{y\in\yy} k_{\suby}(y,y)$. We denote by $C= \EE_{x \sim \rho}~ \phi(x)\otimes\phi(x)$ the covariance operator over input data sampled from $\rho$, and by $\opnorm{C}$ its operator norm, namely its largest singular value. Moreover, we make the following assumption.

\begin{assumption}\label{asm:gstar-exists} There exists $G_*\in\hy\otimes\hx$ with finite trace norm, $\tnorm{G_*} < +\infty$, such that $\gstar(\cdot) = G_*\phi(\cdot)$ is a minimizer of the risk $\risk$ in \cref{eq:surrogate-risk}.
\end{assumption}
The assumption above requires the ideal solution of the surrogate problem to belong to the space of hypotheses of the learning algorithm. This is a standard requirement in statistical learning theory in order to characterize the excess risk bounds of an estimator \citep[see e.g.][]{Shalev-Shwart}.

\begin{restatable}{theorem}{thmgenboundg}\label{thm:gen_boun_g}
Under \cref{asm:gstar-exists}, let $\yy$ be a compact set,  let $(x_i,y_i)_{i=1}^n$ be a set of $n$ points sampled i.i.d. and let $\ghat(\cdot) = \hat G\phi(\cdot)$ with $\hat G$ the solution of \Cref{eq:ivanov_tn} for $\gamma = \tnorm{G_*}$. Then, for any $\delta>0$ 
\eqal{\label{eq:exc_risk_b_g}
	\rr(\ghat) - \rr(\gstar)~\leq~(\sky+\msf{c})\sqrt{\frac{4\log\frac{\mathsf{r}}{\delta}}{n}} + O(n^{-1}),
}
with probability at least $1-\delta$, where 
\eqal{\label{eq:c_for_tn}
 \msf{c} = 2\skx\opnorm{C}^{1/2}\tnorm{G_*}^2  + \skx \risk(g_*)\tnorm{G_*},
}
with  $\mathsf{r}$ a constant not depending on $\delta,n$ or $G_*$.
\end{restatable} 
The proof is detailed in \Cref{sec:vv}. The two main ingredients are: $i)$ the boundedness of the trace norm of $G_*$, which allows us to exploit the duality between trace and operator norms; $ii)$ recent results on Bernstein's inequalities for the operator norm of random operators between separable Hilbert spaces \citep{minsker2017some}. 

We care to point out that previous results are available in the following settings:  \citep{Bach08} shows the convergence in distribution for the trace norm estimator to the minimum risk and \citep{koltchinskii2011} shows excess risk bounds in high probability for an estimator which leverages previous knowledge on the distribution (e.g. matrix completion problem). Both \citep{Bach08} and \citep{koltchinskii2011} are devised for finite dimensional settings. To our knowledge, this is the first work proving excess risk bounds in high probability for trace norm regularized least squares. Note that the relevance of \cref{thm:gen_boun_g} is not limited to structured prediction but it can be also applied to problems such as collaborative filtering with attributes \citep{Abernethy2008}.

\paragraph{Discussion} We now discuss under which conditions trace norm (TN) regularization provides an advantage over standard the Hilbert-Schmidt (HS) one. We refer to \cref{sec:vv} for a more in-depth discussion on the comparison between the two estimators, while addressing here the key points.

For the HS estimator, excess risk bounds can be derived by imposing the less restrictive assumption that $\nor{G_*}_\hs<+\infty$. A result analogous to \cref{thm:gen_boun_g} can be obtained (see \cref{sec:vv}), with constant $\mathsf{c}$ 
\eqals{
 \mathsf{c} ~=~ \skx( \skx{+}\opnorm{C}^{\frac{1}{2}})~\nor{G_*}_{\hs}^2  ~+~ \skx\risk(g_*)\nor{G_*}_{\hs}.
}
This constant is structurally similar to the one for TN (with $\nor{\cdot}_\hs$ appearing in place of $\tnorm{\cdot}$), plus the additional term $\skx^2\nor{G}_{\hs}^2$.

We first note that if $\nor{G_*}_\hs\ll\tnorm{G_*}$, the bound offers no advantage with respect to the HS counterpart. Hence, we focus on the setting where $\nor{G_*}_\hs$ and $\tnorm{G_*}$ are of the same order. This corresponds to the relevant scenario where the multiple outputs/tasks encoded by $G_*$ are (almost) linearly dependent. In this case, the constant $\msf c$ associated to the TN estimator can potentially be significantly smaller than the one for HS: while for TN the term $\tnorm{G_*}^2$ is mitigated by $\opnorm{C}^{1/2}$, for HS the corresponding term $\nor{G_*}_\hs$ is multiplied by $(\skx + \opnorm{C}^{1/2})$. Note that the operator norm is such that $\opnorm{C}^{1/2}\leq \skx$ but can potentially be \textit{significantly} smaller than $\skx$. For instance, when $\xx=\R^d$, $k_{\subx}$ is the linear kernel and training points are sampled uniformly on the unit sphere, we have $\skx = 1$ while $\opnorm{C}^{1/2}=\frac{1}{\sqrt{d}}$. 

In summary, trace norm regularization allows to leverage structural properties of the data distribution {\em provided that the output tasks are related}. This effect can be interpreted as the process of ``sharing'' information among the otherwise independent learning problems. A similar result to \cref{thm:gen_boun_g} was proved in \citep{PontilM13} for Lipschitz loss functions (and $\hy$ finite dimensional). We refer to such work for a more in-depth discussion on the implications of the link between trace norm regularization and operator norm of the covariance operator.

\paragraph{Excess Risk Bounds for $\hat f$} By combining \cref{thm:gen_boun_g} with the comparison inequality for the SELF framework (see \cref{thm:fisher}) we can immediately derive excess risk bounds for the structured prediction estimators $\hat f = \decoding \circ \hat g$. 
\begin{restatable}{corollary}{thmgenboundf}\label{thm:gen_boun_f}
Under the same assumptions and notation of \cref{thm:gen_boun_g}, let $\ell$ be a SELF loss and $\fhat=\decoding\circ\ghat:\xx\rightarrow \yy$. Then, for every $\delta>0$, with probability not less than $1-\delta$ it holds that
\eqals{
	\ee(\hat{f}) - \ee(f_*)~\leq~ \closs\sqrt[4]{\frac{\mathsf{4}(\sky+\msf{c})^2\log\frac{\msf{r}}{\delta}}{n}} + O(n^{-\frac{1}{2}})
}
where $\msf{c}$ and $\mathsf{r}$ are the same constants of \cref{thm:gen_boun_g} and $\closs$ is as in \cref{thm:fisher}.
\end{restatable}  
The result above provides comparable learning rates to those of the original SELF estimator \citep{CilibertoRR16}. However, since the constant $\msf c$ corresponds to the one from \cref{thm:gen_boun_g}, whenever trace norm regularization provides an advantage with respect to standard Hilbert-Schmidt regularization on the surrogate problem, such improvement is directly inherited by $\fhat$.

\subsection{Multitask Learning}

So far we have studied trace norm regularization when learning the multiple $g_t$ in \cref{eq:independent-problems} within a vector-valued setting, namely where for any input sample $x_i$ in training we observe {\em all} the corresponding outputs $\psi_t(y_i)$. This choice was made mostly for notational purposes and the analysis can be extended to the more general setting of nonlinear multitask learning, where separate groups of surrogate outputs could be provided each with its own dataset.
We give here a brief summary of this setting and our results within it, while postponing all details to \cref{sec:appe_mtl}. 


Let $T$ be a positive integer. In typical multitask learning (MTL) settings the goal is to learn multiple functions $f_1,\dots,f_T:\xx\to\yy$ jointly. While most previous MTL methods considered how to enforce linear relations among tasks, \citep{ciliberto2017consistent} proposed a generalization of SELF framework to address nonlinear multitask problems (NL-MTL). In this setting, relations are enforced by means of a constraint set $\conset\subset \yy^T$ (e.g. a set of nonlinear constraints that $f_1,\dots,f_T$ need to satisfy simultaneosly). The goal is to minimize the \textit{multi-task excess risk} 
\eqals{
\min_{f:\xx\rightarrow\conset}\ee_T(f),\qquad\qquad\ee_T(f)=\frac{1}{T}\sum_{t=1}^T\int_{\xx\times\R}\ell(f_t(x),y)d\rho_t(x,y),
}
where the $\rho_t$ are unknown probability distributions on $\xx\times\yy$, observed via finite samples ${(x_{it},y_{it})}_{i=1}^{n_t}$, for $t=1,\dots,T$. The NL-MTL framework interprets the nonlinear multitask problem as a structured prediction problem where the constraint set $\mathcal{C}$ represents the ``structured'' output. Assuming $\ell$ to be SELF with space $\hy$ and coding $\psi$, the estimator $\hat{f}:\xx\rightarrow\conset$ then is obtained via the NL-MTL decoding map $\decoding_T$ 
\eqal{
\label{eq:dec_mtl}
\hat{f}(x)=\decoding_T(\hat{g}(x)):=\argmin_{c\in\conset}\sum_{t=1}^T\sp{\psi(c_t)}{V\hat{g}_t(x)},
}
where each $\ghat_t(\cdot) = G_t\phi(\cdot):\xx\to\hy$ is learned independently via surrogate ridge regression like in \cref{eq:ghat}.

Similarly to the vector-valued case of \cref{eq:independent-problems}, we can ``aggregate'' the operators $G_t\in\hx\otimes\hy$ in a single operator $G$, which is then learned by trace norm regularization as in \cref{eq:trace norm-problem} (see \cref{sec:appe_mtl} for a rigorous definition of $G$). Then, a result analogous to \cref{thm:loss-trick} holds for the corresponding variational formulation of such problem, which guarantees the loss trick to hold as well (see \Cref{subsec:loss_tricks2} for the details of the corresponding version of \cref{alg:trace norm}).


Also in this setting we study the theoretical properties of the low-rank structure prediction estimator obtained from the surrogate Ivanov regularization
\eqal{\label{eq:hatg_surr_mlt}
\hspace{-.26truecm}\hat{G} = \argmin_{\tnorm{G}\leq\gamma}\frac{1}{T}\sum_{t=1}^T \frac{1}{n_t}\sum_{i=1}^{n_t}\nor{G_t\phi(x_{it}) - \psi(y_{it})}_{\hy}^2.
}
We report the result characterizing the excess risk bounds for $\hat G$ (see \cref{thm:complete_thm_mtl_g} for the formal version). Note that in this setting the surrogate risk $\risk_T$ of $G$ corresponds to the average least-squares surrogate risks of the individual $G_t$. In the following we denote by $\bar{C} = \frac{1}{T}\sum_{t=1}^T C_t$ the average of the input covariance operators $C_t = \EE_{x\sim\rho_t} \phi(x)\otimes\phi(x)$ according to $\rho_t$. 

\begin{restatable}[Informal]{theorem}{thmboundsgmlt}\label{thm:bounds_g_mlt}
Under \cref{asm:gstar-exists}, let $\{x_{it},y_{it}\}_{t=1}^{n}$ be independently sampled from $\rho_t$ for $t=1,\dots,T$. Let $\hat g(\cdot) = \hat G\phi(\cdot)$ with $G$ minimizer of \cref{eq:hatg_surr_mlt}. Then for every $\delta>0$, with probability at least $1-\delta$, it holds that
\eqals{
	\risk_T(\hat{g}) - \risk_T(g_*)\leq\sqrt{\frac{\mathsf{2c'}\log\frac{T\mathsf{r}'}{\delta}}{Tn}} + O((nT)^{-1}),
}
where the constant $\mathsf{c}'$ depends on $\tnorm{G_*}$, $\opnorm{\bar{C}}^{1/2}$, $\risk_T(g_*)$ and $\mathsf{r}'$ is a constant independent of $\delta, n,T,G_*$.
\end{restatable}
Here the constant $\msf c'$ exhibits an analogous behavior to $\msf c$ for \cref{thm:gen_boun_g} and can lead to significant benefits in the same regimes discussed for the vector-valued setting. Moreover, also in the NL-MTL setting we can leverage a comparison inequality similar to \cref{thm:fisher}, with constant $\msf{q}_{\mathcal{C},\ell,T}$ from \citep[Thm. 5][]{ciliberto2017consistent}. As a consequence, we obtain the excess risk bound for our MTL estimator $\hat f = \decoding_T \circ \hat g$ of the form
\eqals{
    \ee(\hat{f}) - \ee(f_*)\leq \msf{q}_{\mathcal{C},\ell,T}\sqrt[4]{\frac{\mathsf{c}'\log\frac{T\msf r'}{\delta}}{nT}} + O(n^{-\frac{1}{2}}).
}
The constant $\msf{q}_{\mathcal{C},\ell,T}$, encodes key structural properties of the constraint set $\mathcal{C}$ and it was observed to potentially provide significant benefits over {\em linear} MTL methods (see Ex. 1 in the original NL-MTL paper). Since $\msf{q}_{\mathcal{C},\ell,T}$ is appearing as a multiplicative factor with respect to  $\msf c'$, we could expect our low-rank estimator to provide even further benefits over standard NL-MTL by combining the advantages provided by the {\em nonlinear relations} between tasks and the {\em low-rank relations} among the surrogate outputs.

\section{Experiments}\label{sec:exps}

We evaluated the empirical performance of the proposed method on ranking applications, specifically the \textit{pairwise} ranking setting considered in \citep{duchi2010consistency,Furnkranz:2003,HULLERMEIER2008}. Denote by $\mathcal{D}=\{d_1,\dots,d_N\}$  the full set of {\em documents} (e.g. movies) that need to be ranked. Let $\xx$ be the space of {\em queries} (e.g. users) and assume that for each query $x\in\xx$, a subset of the set of the associated \textit{ratings} $\mathbf{y}=\{y_1,\dots,y_N\}$ is given, representing how relevant each document is with respect to the query $x$. Here we assume each label $y_i\in\{0,\dots,K\}$ with the relation $y_i>y_j$ implying that $d_i$ is more relevant than $d_j$ to $x$ and should be assigned a higher rank. 

We are interested in learning a $f:\xx\to\{1,\dots,N\}^N$, which assigns to a given query $x$ a rank (or ordering) of the $N$ object in $\mathcal{D}$. We measure errors according to the (weighted) \textit{pairwise loss} 
\eqal{
\label{eq:pairwise}\ell(f(x),\mathbf{y})=\sum_{\substack{i=1 \\ j>i}}^N(y_i-y_j)~\textnormal{sign}(f_j(x)-f_i(x)),
}
with $f_i(x)$ denoting the predicted rank for $d_i$. Following \citep{ciliberto2017consistent}, learning to rank with a pairwise loss can be naturally formulated as a nonlinear multitask problem and tackled via structured prediction. In particular we can model the relation between each pair of documents $(d_i,d_j)$ as a function (task) that can take values $1$ or $-1$ depending on whether $d_i$ is more relevant than $d_j$ or vice-versa (or $0$ in case they are equivalently relevant). Nonlinear constraints in the form of a constraint set $\mathcal{C}$ need to be added to this setting in order to guarantee coherent predictions. This leads to a decoding procedure for \cref{eq:dec_mtl} that amounts to solve a minimal feedback arc set
problem on graphs \citep{slater1961inconsistencies}. 


~\newline\noindent We evaluated our low-rank SELF learning algorithm on the following datasets:

\begin{itemize}
\item {\bf Movielens.} We considered Movielens 100k ($ml100k$)\footnote{\url{http://grouplens.org/datasets/movielens/}}, which consists of ratings (1 to 5) provided by $943$ users for a set of $1682$ movies, with a total of $100,000$ ratings available. Additional features for each movie, such as the year of release or its genre, are provided.

\item {\bf Jester.} The Jester\footnote{\url{http://goldberg.berkeley.edu/jester-data/}} datasets consist of user ratings of $100$ jokes where ratings range from $-10$ to $10$. Three datasets are available: $jester1$ with $24,983$ users $jester2$ with $23,500$ users and $jester3$ with $24,938$.

\item {\bf Sushi.} The Sushi\footnote{\url{http://www.kamishima.net/sushi/}} dataset consists of ratings provided by $5000$ people on $100$ different types of sushi. Ratings ranged from $1$ to $5$ and only $50,000$ ratings are available. Additional features for users (e.g. gender, age) and sushi type (e.g. style, price) are provided.
\end{itemize}

\noindent We compared our approach to a number of ranking methods: MART \citep{friedman2001greedy}, RankNet \citep{burges2005learning}, RankBoost \citep{freund2003efficient}, AdaRank \citep{xu2007adarank}, Coordinate Ascent \citep{metzler2007linear}, LambdaMART \citep{wu2010adapting}, ListNet, and Random Forest. For all the above methods we used the implementation provided by RankLib\footnote{\url{https://sourceforge.net/p/lemur/wiki/RankLib/}} library. We also compared with the SVMrank \citep{joachims2006training} approach using the implementation made available online by the authors. Finally, we evaluated the performance of the original SELF approach in \citep{ciliberto2017consistent} (SELF + $\nor{\cdot}_{\hs}$). For all methods we used a linear kernel on the input and for each dataset we performed parameter selection using $50\%$ of the available ratings of each user for training, $20\%$ for validation and the remaining for testing.


%
%

\begin{table*}[!t]
\centering
\scriptsize 

\begin{center}
\tabcolsep=0.11cm
\begin{tabular}{rlllll}
\toprule
        & \multicolumn{1}{c}{$\mathbf{ml100k}$} & \multicolumn{1}{c}{$\mathbf{jester1}$} & \multicolumn{1}{c}{$\mathbf{jester2}$} & \multicolumn{1}{c}{$\mathbf{jester3}$} & \multicolumn{1}{c}{$\mathbf{sushi}$}\\
\midrule

{\bf MART}              & $0.499 ~(\pm 0.050)$     & $0.441 ~(\pm 0.002)$    & $0.442 ~(\pm 0.003)$   & $0.443 ~(\pm 0.020)$   & $0.477  ~(\pm0.100)$\\
{\bf RankNet}           & $0.525 ~(\pm0.007)$     & $0.535 ~(\pm 0.004)$    & $0.531 ~(\pm0.008)$   & $0.511 ~(\pm 0.017)$   & $0.588  ~(\pm0.005)$\\
{\bf RankBoost}         & $0.576 ~(\pm0.043)$     & $0.531 ~(\pm0.002)$    & $0.485 ~(\pm0.061)$   & $0.496 ~(\pm 0.010)$   & $0.589 ~(\pm0.010)$\\
{\bf AdaRank}           & $0.509 ~(\pm0.007)$     & $0.534 ~(\pm0.009)$    & $0.526 ~(\pm0.001)$   & $0.528 ~(\pm 0.015)$   & $0.588 ~(\pm0.051)$\\
{\bf Coordinate Ascent} & $0.477 ~(\pm0.108)$     & $0.492 ~(\pm0.004)$    & $0.502 ~(\pm0.011)$   & $0.503 ~(\pm 0.023)$   & $0.473 ~(\pm0.103)$\\
{\bf LambdaMART}        & $0.564 ~(\pm0.045)$     & $0.535 ~(\pm0.005)$    & $0.520 ~(\pm0.013)$   & $0.587 ~(\pm 0.001)$   & $0.571 ~(\pm0.076)$\\
{\bf ListNet}           & $0.532 ~(\pm0.030)$     & $0.441 ~(\pm0.002)$    & $0.442 ~(\pm0.003)$   & $0.456 ~(\pm 0.059)$   & $0.588 ~(\pm0.005)$\\
{\bf Random Forests}    & $0.526 ~(\pm0.022)$     & $0.548 ~(\pm0.001)$    & $0.549 ~(\pm0.001)$   & $0.581 ~(\pm 0.002)$   & $0.566 ~(\pm0.010)$\\
{\bf SVMrank}           & $0.513 ~(\pm0.009)$     & $0.507 ~(\pm 0.007)$    & $0.506 ~(\pm 0.001)$   & $0.514 ~(\pm 0.009)$   & $0.541 ~(\pm0.005)$\\
{\bf SELF + $\nor{\cdot}_\hs$}  & $0.312 ~(\pm 0.005)$     & $0.386 ~(\pm 0.005)$   & $0.366 ~(\pm 0.002)$   & $0.375 ~(\pm 0.005)$   & $0.391 ~(\pm 0.003)$\\
{\bf (Ours) SELF + $\nor{\cdot}_{*~}$}  & $\mathbf{0.156 ~(\pm 0.005)}$     & $\mathbf{0.247 ~(\pm 0.002)}$   & $\mathbf{0.340 ~(\pm 0.003)}$   & $\mathbf{0.343 ~(\pm 0.003)}$   & $\mathbf{0.313 ~(\pm 0.003)}$\\
\hline
\end{tabular}
\caption{Ranking error of benchmark approaches and our proposed method on five ranking datasets.}\label{tab:results}
\end{center}
\end{table*}

\paragraph{Results} Table \ref{tab:results} reports the average performance of the tested methods across five independent trials. Prediction errors are measured in terms of the pair-wise loss in \cref{eq:pairwise}, normalized between $0$ and $1$. A first observation is that the performance of both SELF approaches significantly outperform the competitors. This is in line with the observations in \citep{ciliberto2017consistent}, where the nonlinear MTL approach based on the SELF framework already improved upon state of the art ranking methods. Moreover, our proposed algorithm, which combines ideas from structured prediction and multitask learning, achieves an even lower prediction error on all datasets. This supports the idea motivating this work that leveraging the low-rank relations can provide significant advantages in practice. 

\section{Conclusions}\label{sec:conclusions}

This work combines structured prediction methods based on surrogate approaches with multitask learning techniques. In particular, building on a previous framework for structured prediction we derived a trace norm regularization strategy that does not require explicit knowledge of the coding function. This led to a learning algorithm that can be efficiently applied in practice also when the surrogate space is large or infinite dimensional. We studied the generalization properties of the proposed estimator based on excess risk bounds for the surrogate learning problem. Our results on trace norm regularization with least-squares loss are, to our knowledge, novel and can be applied also to other settings such as collaborative filtering with side information. Experiments on ranking applications showed that leveraging the relations between surrogate outputs can be beneficial in practice. 

A question opened by our study is whether other multitask regularizers could be similarly adopted. As mentioned in the paper, even well-established approaches, such as those based on incorporating in the regularizer prior knowledge of the similarity between tasks pairs, do not always extend to this setting. Further investigation in the future will be also devoted to consider alternative surrogate loss functions to the canonical least-squares loss, which could enforce desirable tasks relations between the surrogate outputs more explicitly.


\bibliographystyle{abbrvnat}

\bibliography{biblio}


\onecolumn
\appendix

\setlength{\parskip}{\baselineskip}
\setlength\parindent{0pt}

\section*{\huge Appendix}

The supplementary material is organized as follows:
\begin{itemize}  
    \item In \cref{sec:loss_tricks} we show how the loss trick for both the vector-valued and multitask SELF estimator is derived. 
    \item In \cref{sec:vv} we carry out the theoretical analysis for trace norm estimator in the vector-valued setting. 
    \item In \cref{sec:appe_mtl} we prove the theoretical results characterizing the generalization properties of the SELF multitask estimator.
    \item In \cref{sec:auxiliary_lemma} we recall some results that are used in the proofs of previous sections.
    
    \item In \cref{sec:relation-between-tikhonov-ivanov} more details on the equivalence between Ivanov and Tikhonov regularization are provided.
\end{itemize}

\section{Loss Trick(s)}\label{sec:loss_tricks}

In this section we discuss some aspects related to the loss trick of the SELF framework when considering different vector-valued or MTL estimators.


  \subsection{Loss Tricks with Matrix Factorization}\label{subsec:loss_tricks1}
  In this section we provide full details of the loss trick for trace norm regularization partly discussed in Section \ref{sec:trace norm}.
  To fix the setting, recall that we are interested in studying the following surrogate problem
\eqal{\label{eq:surrpb_tracen}
	\min_{G\in\hy\otimes\hx}\frac{1}{n}\sum_{i=1}^n \nor{G\phi(x_i) - \psi(y_i)}_{\hy}^2 + \lambda \tnorm{G}.
}
  

\losstrick*
\begin{proof}
We show the proof in the finite dimensional setting first and then note how it is valid in the infinite dimensional case as well.
Assume  $\xx = \R^d$ and $\hy = \R^T$. Let $\{(x_i,y_i)_{i=1}^n$ be the training set and denote by $X$ the $\R^{n\times d}$ matrix containing  the training inputs $x_i$, $i=1,\dots,n$ and $Y$ the $\R^{n \times T}$ matrix whose rows are $\psi(y_i)$, $i=1,\dots,n$. Denote by $K_{\subx}$ the matrix $X X^\top$ and by $K_{\suby}$ the matrix $YY^\top$. 

Using the  variational form of trace norm, problem \eqref{eq:surrpb_tracen} can be rewritten as
\eqal{ \min_{ A \in \R^{d\times r}, B\in \R^{T\times r}}\frac{1}{n}\nor{Y - X A B^\top}^2+\lambda ( \nor{A}_{\hs}^2+\nor{B}_{\hs}^2),}
where $r\in\mathbb{N}$ in a further hyperparameter of the problem. In the following we will absorb the factor $1/n$ in the hyperparameter $\lambda$.

\noindent We first show that starting gradient descent algorithm with $A_0:= X^\top M_0$ for some matrix $M_0\in \R^{n\times r}$ and $B_0:= Y^\top N_0$ for some matrix $N_0\in \R^{n\times r}$, then at every iteration $A_k:= X^\top M_k$ and $B_k:= Y^\top N_k$. \\

\noindent Let us set \[ \mathcal{L}(A,B):= \nor{Y - X A B^\top}^2+\lambda (\nor{A}_{\hs}^2+\nor{B}_{\hs}^2);\] the gradients of $\mathcal{L}$ with respect to $A$ and $B$ are given by \begin{itemize}
\item[1)] $\nabla_A\mathcal{L}(A,B) = X^\top(XAB^\top-Y)B + \lambda A$
\item[2)] $\nabla_B\mathcal{L}(A,B)=( X A B^\top-Y )^\top X A +\lambda B$. 
\end{itemize}
 We show that $A_k:= X^\top M_k$ and $B_k:= Y^\top N_k$ by induction. Assume it is true for $k$ and show it holds for $k+1$; denoting by $\nu$ the stepsize, we have
  \eqals{
A_{k+1} &= A_k - \nu \nabla_A\mathcal{L}(A_k,B_k) = A_k - \nu (X^\top(X A_k B_k^\top-Y)B_k + \lambda A_k)\\
&= X^\top M_k   - \nu( X^\top X X^\top M_k B_k^\top B_k - X^\top Y B_k) - \nu \lambda X^\top M_k\\
&= X^\top ((1-\lambda \nu)M_k - \nu (K_{\subx} M_k B_k^\top B_k-Y B_k)\\
& = X^\top\big((1-\lambda \nu)M_k - \nu ( K_{\subx} M_k N_k^\top K_{\suby} N_k-K_{\suby} N_k )\big),
}
and hence $A_{k+1} = X^\top M_{k+1}$ 
\eqal{\label{eq:updateM_k} M_{k+1} = (1-\lambda \nu)M_k - \nu \big(K_{\subx} M_k N_k^\top K_{\suby} N_k-K_{\suby} N_k\big).}\\
\noindent As for $B$, assume $B_{k}=Y^\top N_{k}$: 
\eqals{
B_{k+1} &= B_k - \nu \nabla_B \mathcal{L}(A_k,B_k)\\
 &= B_k - \nu ((X A_k B_k^\top-Y )^\top X A_k +\lambda B_k)\\
 &= Y^\top N_k - \nu  (Y^\top N_k A_k^\top X^\top X A_k-Y^\top X A_k) - \nu \lambda Y^\top N_k\\
 &= Y^\top ((1-\lambda \nu)N_k - \nu (N_k (K_{\subx} M_k)^\top K_{\subx} M_k- K_{\subx} M_k))}
and hence $B_{k+1} = Y^\top N_{k+1}$ with 
\eqal{\label{eq:updateN_k} N_{k+1} = (1-\lambda \nu)N_k - \nu ( N_k M_{k}^\top K^\top_\xx K_{\subx}  M_{k}-K_{\subx} M_{k}).}

Then, denote by $M$ and $N$ the limits of $M_k$ and $N_k$. 
Given a new $x$, the estimator is  \[ \hat{g}_k(x) = x X^\top M_k N_k^\top Y.\] Expanding the product we can rewrite \[
\hat{g}_k(x) = \sum_{i=1}^n \alpha^{\mathsf{tn}}_i(x) \psi(y_i), \qquad  \alpha^{\mathsf{tn}}(x)=N_k M_k^\top X x^\top = N_k M_k^\top v_x,
 \]
where $v_x=X x^\top\in\R^n$.
 Let $\decoding$ be the decoding map defined by \[ 
\decoding(h) =  \argmin_{y\in\yy} \sp {\psi(y)} {Vh}.
 \]
 Then \[
  \hat{f}_k(x)= \decoding\circ \hat{g}_k(x)=  \argmin_{y\in\yy}\sum_{i=1}^n \alpha^{\mathsf{tn}}_i(x)\sp {\psi(y)} {V \psi(y_i)} =   \argmin_{y\in\yy} \sum_{i=1}^n \alpha^{\mathsf{tn}}_i(x)\ell (y, {y_i}). 
 \]
 
 Note that in order to obtain the estimator $\hat{g}_k$, only the access to $M_k$ and $N_k$ is needed. Also, examining the updates for $M_k$ and $N_k$ outlined in \eqref{eq:updateM_k} and \eqref{eq:updateN_k} we note that the data are accessed through $K_{\subx}$ and $K_{\suby}$ only, which are kernels on input and output respectively. This leads to  a direct extension of the argument in the infinite dimensional setting, where the RKHSs $\hx$ and $\hy$ on input and output spaces are infinite dimensional Hilbert spaces. 
\end{proof}

\subsection{Loss Trick in the Multitask Setting}\label{subsec:loss_tricks2}

We now turn to the \textbf{multitask} case. 

We recall the surrogate problem with trace norm regularization, i.e. 
\eqal{\label{eq:surr_mlt_tnorm}\min_{G\in\R^T\otimes \hx}\frac{1}{T}\sum_{t=1}^T\frac{1}{n_t} \sum_{i=1}^{n_t} \nor{G_t\phi(x_{it}) - \psi(y_{it})}^2 + \lambda\tnorm{G}.}
\begin{restatable}{proposition}{losstrickmlt}\label{prop:loss_trick_mlt}
Let $k:\xx\times \xx\rightarrow \R$ be a reproducing kernel with associated RKHS $\hx$. Let $\hat{g}=\hat{G}\phi(\cdot)$ be the solution of problem \eqref{eq:surr_mlt_tnorm}, denote by $\hat{g}_t$, $t=1,\dots,T$ its components. Then the loss trick applies to this setting, i.e. the estimator $\hat{f} = \decoding\circ \hat{g}$ with $\decoding_T$ as in \Cref{eq:dec_mtl},  is equivalently written as \eqal{
\hat{f}(x)=\argmin_{c\in\conset}\sum_{t=1}^T\sum_{i=1}^{n_t}\alpha^{\mathsf{tn}}_{it}(x)\ell(c_t,y_{it}),
}
for some coefficients $\alpha_{it}$ which are derived in the proof below.
\end{restatable}
\begin{proof}
Assume  $\xx = \R^d$ and $\hy = \R^T$ for the sake of clarity, so that $G\phi(x)=Gx$. For any $t=1,\dots,T$, let $\{(x_{it},y_{it})\}_{i=1}^{n_t}$, be the training set for the $t^{th}$ task. 

Denote by $X\in \R^{n\times d}$ the matrix containing  the training inputs $x_{it}$, and by $Y \in \R^{n \times T}$ the matrix whose rows are $\psi(y_{it})$; denote by $X_t$ the $n_t \times d$ matrix containing training inputs of the $t^{th}$ task and by $Y_t$ the $n_t\times 1$ vector with entries $\psi(y_{it})$ $i=1,\dots,n_t$. 
We rewrite \eqref{eq:surr_mlt_tnorm} using the variational form of the trace norm:
\eqal{ \min_{ A \in \R^{d\times r}, B\in \R^{T\times r}}\nor{Q \odot(Y - X A B^\top)}^2+\lambda (\nor{A}_{\hs}^2+\nor{B}_{\hs}^2),}
where $r\in\mathbb{N}$ is now a hyperparameter and  $Q$ is a mask which contains zeros in correspondence of missing data.  The expression above is also equivalent to 
\eqals{
\min_{ A \in \R^{d\times r}, B\in \R^{T\times r}}\frac{1}{T}\sum_{t=1}^T\bigg( \frac{1}{n_t} \nor{X_t A B_t - Y_t}^2+\lambda(\nor{B_t}_{\hs}^2 + \nor{A}_{\hs}^2)\bigg),}
where $B_t$ denotes the $t^{th}$ row of $B$, i.e. $B_t$ is a $1\times r$ vector. Thanks to this split, we can update $B$ by updating its rows separately, via  (we omit factors 2 which would come from derivatives)
\eqals{B_{t,k+1} = B_{t,k} - \nu \big( n_{t}^{-1}(B_{t,k} A_k^\top X_t^\top - Y_t^\top)X_t A_k+\lambda B_{t,k}\big).}
Initialising $B_{t,0} = Y_t^\top N_{t,0}$ for some matrix $N_{t,0}\in\R^{n_t \times r}$, gradient descent updates preserve the structure, and for each $k$, $B_{t,k} = Y_t^\top N_{t,k}$. Indeed, 
\eqals{B_{t,k+1}  &= B_{t,k} - \nu\big( n_{t}^{-1}(B_{t,k} A_k^\top X_t^\top - Y_t^\top)X_t A_k+\lambda B_{t,k}\big)\\
& = Y_t^\top N_{t,k} - \nu \big(  n_{t}^{-1}(Y_t^\top N_{t,k} A_k^\top X_t^\top - Y_t^\top)X_t A_k+\lambda Y_t^\top N_{t,k}   \big)\\
&= Y_t^\top\big( (1-\nu\lambda)N_{t,k} -\nu  n_{t}^{-1}(N_{t,k} A_k^\top X_t^\top X_t A_k  - X_t A_k)\big)\\
& = Y_t^\top N_{t,k+1} }
where \eqals{N_{t,k+1} = (1-\nu\lambda)N_{t,k} -\nu  n_{t}^{-1}(N_{t,k} A_k^\top X_t^\top X_t A_k  - X_t A_k).}

Let us now focus on updates of $A$, and then combine the two. Set \[ \mathcal{L}(A,B):= \nor{Q\odot(Y - X A B^\top)}^2+\lambda \bigg( {\nor{A}_{\hs}^2+\nor{B}_{\hs}^2}\bigg);\]
Note that 
the gradient with respect to $A$ reads as  $\nabla_A\mathcal{L}(A,B) =X^\top(Q\odot(XAB^\top-Y))B + \lambda A$.
Hence, initialising $A_0 = X^\top M_0$, each iterate $A_k$ has the form $X^\top M_k$ and it is possible to perform updates on $M_k$ only as in the proof of  \cref{thm:loss-trick}, via 
\eqals{M_{k+1} = (1-\lambda \nu)M_k -\big( (Q \odot( X X^\top M_k B_k^\top-Y) B_k \big).} 
Let us analyse the term $(Q \odot( X X^\top M_k B_k^\top-Y)) B_k $: leveraging the structure of the mask, 
\eqal{(Q \odot( X X^\top M_k B_k^\top-Y)) B_k = [S_1^\top, \dots, S_T^\top]^\top}
where $S_t$ is a $n_t \times r$ matrix equal to 
\eqals{S_t = X_t X^\top M_k B_{t,k}^\top B_{t,k} = X_t X^\top M_k N_{t,k}^\top Y_tY_t^\top N_{t,k}. }
At convergence, we will have $A = X^\top M$ and $B_{t} = Y_t^\top N_{t}$ for $t=1,\dots,T$.
Hence, the $t^{th}$ component of the estimator is given by 
\eqals{\hat{g}_t(x) = x A B_t^\top = x X^\top M N_{t}^\top Y_t= \sum_{i=1}^{n_t}\alpha^{\mathsf{tn}}_{it}(x)\psi(y_{it}),\qquad \alpha^{\mathsf{tn}}_t(x) =  N_t M^\top X x^\top.}
Then, the estimator $\hat{f}_N$, with $N = (n_1,\dots,n_T)$  is given by
 \eqals{\hat{f}_N(x) =\argmin_{c\in\conset} \sum_{t=1}^T \sp{c_t}{V\hat{g}_t(x)} = \argmin_{c\in\conset} \sum_{t=1}^T \sum_{i=1}^{n_t} \alpha^{\mathsf{tn}}_{it}(x)\sp{c_t}{V \psi(y_{it})} = \argmin_{c\in\conset} \sum_{t=1}^T \sum_{i=1}^{n_t}\alpha^{\mathsf{tn}}_{it}(x) \ell(c_t,y_{it}),}
 and hence the loss trick holds.
\end{proof}

\subsection{Remark on the Lack of Loss Trick for Regularizers via Positive Semidefinite Operator}\label{sec:C0}
  Assume $\xx=\R^d$, $\hy= \R^T$ and let $Y$ be the $n\times T$ matrix containing $\psi(y_i)$ in its rows. Given $A\in\R^{T\times T}$ symmetric positive definite, the surrogate problem with regularizer $\tr(GAG^\top)$ reads as 
  \eqals{
  	\frac{1}{n}\nor{Y - XG}^2 + \lambda\tr(GA G^\top).
  }
 We omit the factor $1/n$ as it is does not affect what follows. The problem above has  the following solution (see for instance \citep{alvarez2012kernels})
 \eqals{\mathsf{vec}(G) = (I \otimes X^\top X +\lambda A\otimes I)^{-1}(I\otimes X^\top)\mathsf{vec}(Y).}
 This can be rewritten as
  \eqals{
\mathsf{vec}(G)   &=  (A^{-1/2} \otimes I)(A^{-1}\otimes X^\top X +\lambda I)^{-1}(A^{-1/2}\otimes X^\top)\mathsf{vec}(Y)\\
&=(A^{-1} \otimes X^\top )(A^{-1}\otimes K + \lambda I)^{-1}\mathsf{vec}(Y),
 }
 where $K= X X^\top$ is the kernel matrix. Setting $\textsf{vec}(M(Y)) = (A^{-1}\otimes K + \lambda I)^{-1}\mathsf{vec}(Y)$,
 \eqals{\mathsf{vec}(G) = (A^{-1}\otimes X^\top)\mathsf{vec}(M(Y)) = \mathsf{vec}(X^\top M(Y) A^{-\top}) = \mathsf{vec}(X^\top M(Y) A^{-1}), }
 since $A$ is symmetric. Then $G = X^\top M(Y) A^{-1}$. 
The decoding procedure yields \eqals{
\hat{f}(x) = \decoding(\hat{g}(x)) = \argmin_{y\in\yy}\sp{Y}{V\hat{g}(x)}= \argmin_{y\in\yy}\sp{Y}{V A^{-1} M(Y)^\top v_x},}
and due to the product $VA^{-1}$ we cannot retrieve the loss function, i.e. the loss trick.

  Now, let us distinguish the following cases 
  \begin{itemize}
  \item[$\mathbf{1.}$]  $\yy$ has finite cardinality;
      \item[$\mathbf{2.}$] $\yy$ has not finite cardinality, $\hy$ is infinite dimensional or  $\psi$  and $V$ are unknown.
  \end{itemize} 
In the \textit{first} case, let us set $\mathsf{N} = \{1,\dots,\abs{\yy} \}$ and $\hy = \R^{\abs{\yy}}$. Let $q:\yy\rightarrow \mathsf{N}$ be a one-to-one function and for $y\in\yy$ set $Y= e_{q(y)}$ where $e_i$ denoted the $i^{th}$ element of the canonical basis of $\R^{\abs{\yy}}$. 
Also, set $V\in\R^{\abs{\yy}\times \abs{\yy}}$ the matrix with entries $V_{ij} = \ell(q^{-1}(i), q^{-1}(j))$. Then, since $A$ is a known matrix, $\psi$ and $V$ are defined as above, the estimator $\hat{f}$ can be retrieved despite the lack of loss trick.\\

In the \textit{second} case, it is not clear how to manage the operation $VA^{-1}$ since $V$ is unknown and also, both $V$ and $A$ are bounded operators from an infinite dimensional space to itself. While in the standard SELF framework, the infinite dimensionality is hidden in the loss trick, and there is no need to explicitly deal with infinite dimensional objects, here it appears to be necessary due to the action of $A$.

\section{Theoretical Analysis}\label{sec:vv}
\thmgenboundg*
\begin{proof}
 \noindent We split the error as follows:
 \eqals{\risk(\hat{g}) - \risk(g_*)\leq& \risk(\hat{g}) - \hat{\risk}(\hat{g})+
 \hat{\risk}(\hat{g}) - \hat{\risk}(g_{\gamma *})\\
 +&\hat{\risk}(g_{\gamma *}) -  \risk(g_{\gamma *})+ 
 \risk(g_{\gamma *}) - \risk(g_*).}
Now, by definition of $\hat{g}$ the term $\hat{\risk}(\hat{g}) - \hat{\risk}(g_{\gamma *})$ is negative. Also, denoting by $\rho_{t\mid\xx}$ the marginal on $\xx$ of the probability measure $\rho_t$,
\eqal{\label{eq:part3_of_dec}\risk(g_{\gamma *}) - \risk(g_*)& = \int_{\xx}\nor{G_{\gamma *}\phi(x) - G_{*}\phi(x)}_{\hy}^2\,d\rho_{\xx}(x) = \inf_{\{ G \in\GG_\gamma \}}\nor{G\phi(x) - G_{*}\phi(x)}^2_{L^2(\rho_{\xx})}\\
&\leq \nor{\big(\frac{\gamma}{\tnorm{G_{*}}}\big)G_{*}\phi(x) - G_{*}\phi(x)}^2_{L^2(\rho_{\xx})} \leq \Big( 1-\frac{\gamma}{\tnorm{G_*}}\Big)^2 \nor{G_{*}\phi}_{L^2(\rho_{\xx})}^2\\
&\leq \Big( 1-\frac{\gamma}{\tnorm{G_*}}\Big)^2 \skx^2\nor{G_*}_{\hs}^2\leq (\tnorm{G_*} - \gamma)^2\skx^2.}

\noindent It remains to bound $R_1:=\risk(\hat{g}) - \hat{\risk}(\hat{g})$ and $R_2:=\hat{\risk}(g_{\gamma *}) -  \risk(g_{\gamma *})$. Since
\eqals{R_1+R_2 \leq 2\sup_{G\in\GG_\gamma}\abs{\hat{\risk}(G) - \risk(G)},}
we just have to bound the term on the right hand side.

Denote 
\eqals{
    & C = \EE \phi(x)\otimes\phi(x), \qquad \hat C = \frac{1}{n} \sum_{i=1}^n \phi(x_i)\otimes\phi(x_i) \\
    & Z = \EE \psi(y)\otimes\phi(x), \qquad \hat Z = \frac{1}{n} \sum_{i=1}^n \psi(y_i)\otimes\phi(x_i). 
}
For any operator $G$ in $\hy\otimes \hx$, we have
\eqals{
    \abs{\hat{\risk}(G) - \risk(G)} & = \left| \frac{1}{n} \sum_{i=1}^n \nor{\psi(y_i)-G\phi(x_i)}_\hh^2 - \EE \nor{\psi(y) - G\phi(x)}_\hh^2 \right| \\
    & = \Bigg| \frac{1}{n} \sum_{i=1}^n\left( \scal{G^*G}{\phi(x_i)\otimes\phi(x_i)}_{\hs}  - 2\scal{G}{\psi(y_i)\otimes\phi(x_i)}_{\hs} + \nor{\psi(y_i)}_\hy^2 \right) \\
    & \qquad\qquad - \EE \left(\scal{G^*G}{\phi(x)\otimes\phi(x)}_\hs - 2 \scal{G}{\psi(y)\otimes\phi(x)}_\hs + \nor{\psi(y)}_\hy^2 \right) \Bigg| \\
    & = \left| \scal{G^*G}{\hat C - C}_{\hs} - 2\scal{G}{\hat Z - Z}_{\hs} + \frac{1}{n}\sum_{i=1}^n \nor{\psi(y_i)}_\hy^2 - \EE \nor{\psi(y)}_\hy^2 \right| \\ 
    & \leq \nor{G}_{\hs}^2\opnorm{C - \hat C} + 2 \tnorm{G}\opnorm{Z - \hat Z} + \left|\EE\nor{\psi(y)}_\hy^2 - \frac{1}{n}\sum_{i=1}^n\nor{\psi(y_i)}_\hy^2\right|. 
}
In the last inequality we used that $\tnorm{G^*G}=\nor{G^*G}_{\hs}=\nor{G}_{\hs}^2$ in the first part.
In the following we bound $\opnorm{C - \hat{C}}$ and $\opnorm{Z - \hat{Z}}$, in two different steps.\\

\textbf{STEP 1} Let us start with $\opnorm{C - \hat{C}}$.
We leverage the result in \citep{minsker2017some} on Bernstein's inequality for self adjoint operators, which are recalled in \cref{lem:berst_ineq} below. 
Let us set \eqals{X_i:=(\phi(x_i)\otimes \phi(x_i) - C)/n} and note that $\EE(X_i) = 0$. 
Also, resolving the square we have that 
\eqals{
\EE(X_i^2) = \frac{1}{n^2}\EE(\sp{\phi(x_i)}{\phi(x_i)}\phi(x_i)\otimes \phi(x_i) -2 \phi(x_i)\otimes \phi(x_i) C+ C^2)
= \frac{1}{n^2} \EE(\skx^2\phi(x_i)\otimes \phi(x_i)) - C^2,
}
and hence (we assume $\skx\geq 1$)  \eqals{
\nor{\sum_{i=1}^n \EE X_i^2}\leq \frac{1}{n}(\skx^2\nor{C}_{\rm op} + \nor{C}_{\rm op}^2)\leq \frac{2\skx^2}{n}\opnorm{C}=:\sigma^2,
}

Since $\nor{\phi(x_i)}\leq \skx$ for any $i=1,\dots,n$, we get 

\eqals{ 
\nor{X_i}\leq \frac{\skx^2 + \opnorm{C}}{n}\leq \frac{2\skx^2}{n}:=U.
}

Set \eqals{
\bar{r}_1:= \frac{\tr\left(\sum_{i=1}^n \EE X_i^2\right)}{\opnorm{\sum_{i=1}^n \EE X_i^2}}.
}
Note that the quantity above is the effective rank of $\sum_{i=1}^n \EE X_i^2$.
With $\sigma^2$ and $U$ as above, \cref{lem:berst_ineq} yields 

\eqals{
\opnorm{C-\hat{C}} \leq \frac{4}{n}\Big( \frac{\skx^2}{3}\ln\Big(\frac{14\bar{r}_1}{\delta}\Big)\Big) + \sqrt{\frac{4\skx^2\opnorm{C}}{n}\ln\Big(\frac{14 \bar{r}_1}{\delta}\Big)}
}
with probability greater or equal to $1-\delta$.\\

\textbf{STEP 2} As for $\opnorm{Z - \hat{Z}}$ we proceed in a similar way: let $X_i : = (\psi(y_i)\otimes \phi(x_i) - Z)/n$. 
Then, 
\eqals{\nor{X_i} \leq \frac{\sky\skx + \nor{Z}_{\rm op}}{n}\leq \frac{2\skx\sky}{n}.
}
Also,

\eqal{
\EE X_i^*X_i &=  \frac{1}{n^2} \EE[(\phi(x_i)\otimes \psi(y_i) - Z^*)(\psi(y_i)\otimes \phi(x_i) - Z)] \\&= \frac{1}{n^2}\left( \EE(\sp{\psi(y_i)}{\psi(y_i)}\phi(x_i)\otimes \phi(x_i))-Z^*Z\right)
\preceq  \frac{2}{n^2} \EE(\sp{\psi(y_i)}{\psi(y_i)}\phi(x_i)\otimes \phi(x_i)).
}
Then 
\eqals{
\opnorm{\sum_{i=1}^n \EE X_i^* X_i}\leq \frac{2}{n}\opnorm{\EE(\sp{\psi(y)}{\psi(y)}\phi(x)\otimes \phi(x))}.
}
Applying \cref{lem:bound_ky}, we obtain 
\eqals{
\opnorm{\sum_{i=1}^n \EE X_i^* X_i}\leq \frac{2\skx^2}{n}( \nor{G_*}_{\hs}^2\opnorm{C}+\risk(g_*)).
}

Similarly, 
\eqals{
\EE X_iX_i^* &=  \frac{1}{n^2} \EE[(\psi(y_i)\otimes \phi(x_i) - Z)(\phi(x_i)\otimes \psi(y_i) - Z^*)] \\&= \frac{1}{n^2}\left( \EE(\sp{\phi(x_i)}{\phi(x_i)}\psi{y_i})\otimes \psi(y_i))-ZZ^*\right)
\preceq  \frac{2}{n^2} \EE(\sp{\phi(x_i)}{\phi(x_i)}\psi(y_i)\otimes \psi(y_i))\\
&\preceq  \frac{2}{n^2} \skx^2 \EE(\psi(y_i)\otimes \psi(y_i)).
}
and 
\eqals{
\opnorm{\sum_{i=1}^n \EE X_i X_i^*}\leq \frac{2\skx^2}{n}\opnorm{\EE(\psi(y)\otimes \psi(y))}.
}
Applying \cref{lem:boundnorm_cy}, we conclude 
\eqals{
\opnorm{\sum_{i=1}^n \EE X_i X_i^*}\leq  \frac{2\skx^2}{n}( \nor{G_*}_{\hs}^2\opnorm{C}+\risk(g_*)).
}
Hence both $\opnorm{\sum_{i=1}^n \EE X_i X_i^*}$ and $\opnorm{\sum_{i=1}^n \EE X_i^* X_i}$ are bounded by $\frac{2\skx^2}{n}( \nor{G_*}_{\hs}^2\opnorm{C}+\risk(g_*))$.

Moreover, let 

\eqals{
    \bar r_2 = \max\left(\frac{\tr(\sum_{i=1}^n \EE X_i X_i^*)}{\nor{\sum_{i=1}^n \EE X_i X_i^*}_{\rm op}},\frac{\tr(\sum_{i=1}^n \EE X_i^* X_i)}{\nor{\sum_{i=1}^n \EE X_i^* X_i}_{\rm op}}\right),
}
which corresponds to the maximum between effective ranks of $\sum_{i=1}^n \EE X_i X_i^*$ and $\sum_{i=1}^n \EE X_i^* X_i$.

Bernstein's inequality shown in \citep{minsker2017some} (and recalled in \cref{lem:berst_ineq_2}) gives 
\eqals{
\nor{Z-\hat{Z}}_{\rm op}\leq \frac{4}{n}\Big( \frac{\skx \sky }{3}\ln\Big(\frac{28\bar{r}_2}{\delta}\Big)\Big)+ \sqrt{\frac{2\skx^2(\nor{G_*}_{\hs}^2\opnorm{C}+\risk(g_*))}{n}\ln\Big(\frac{28 \bar{r}_2}{\delta}\Big)}
} 
with probability greater or equal to $1-\delta$. Splitting the second term we see that \eqals{
\nor{Z-\hat{Z}}_{\rm op}\leq \frac{4}{n}\Big( \frac{\skx \sky }{3}\ln\Big(\frac{28\bar{r}_2}{\delta}\Big)\Big)+\left(\nor{G}_{\hs}\opnorm{C}^{\frac{1}{2}}\skx +\skx\sqrt{\risk(g_*)}\right)\sqrt{\frac{2}{n}\ln\Big(\frac{28 \bar{r}_2}{\delta}\Big)}.
}
\textbf{STEP 3.} Finally, by Hoeffding inequality \eqals{
\left|\EE\nor{\psi(y)}_\hy^2 - \frac{1}{n}\sum_{i=1}^n\nor{\psi(y_i)}_\hy^2\right|\leq\sky \sqrt{\ln\Big(\frac{2}{\delta}\Big)\frac{1}{n}}
}
with probability at least $1-\delta$.

\textbf{FINAL STEP.} We have now all the bounds that we need. By taking $\mathsf{r} = \max(\bar{r}_1,\bar{r}_2)$ and performing an intersection bound on the three parts we conclude 
\eqal{\label{eq:bound1} \abs{\hat{\risk}(G) - \risk(G)}\leq \gamma^2\Big(\frac{A}{n}+\frac{B}{\sqrt{n}}\Big) + \gamma\Big(\frac{A'}{n}+\frac{B'}{\sqrt{n}}\Big) + \sky\sqrt{\ln \Big( \frac{2}{\delta}\Big)\frac{1}{n}}} with probability greater or equal than $1-3\delta$, with \eqals{
&A = 4\ln\Big( \frac{28\mathsf{r}}{\delta}\Big)\frac{\skx^2}{3}, \qquad B = (2+\sqrt{2})\skx\opnorm{C}^{\frac{1}{2}}\sqrt{\ln \Big(\frac{28\mathsf{r}}{\delta}}\Big)\\
&A' =  4\ln\Big( \frac{28\mathsf{r}}{\delta}\Big)\frac{\skx\sky}{3}, \qquad B' =\skx\sqrt{2\risk(g_*)}\sqrt{\ln \Big(\frac{28\mathsf{r}}{\delta}\Big).}}

Combining with the approximation error in \cref{eq:part3_of_dec}, we obtain 
\eqals{
\risk(\hat g) - \risk(g_*) \leq \gamma^2\Big(\frac{A}{n}+\frac{B}{\sqrt{n}}\Big) + \gamma\Big(\frac{A'}{n}+\frac{B'}{\sqrt{n}}\Big) + \sqrt{\ln \Big( \frac{2}{\delta}\Big)\frac{\sky^2}{n}} + (\nor {G_*}_*-\gamma)^2\skx^2.}
In principle, starting from the bound above we should optimize with respect to $\gamma$ to find the optimal value, which will be between 0 and $\nor{G_*}_*$.  Here we consider the simpler case where $\gamma = \nor{G_*}_*$.
Isolating the faster terms, the bound above becomes 
\eqals{\risk(\hat{g}) - \risk(g_*)\leq
\frac{\nor{G_*}_{*}^2}{\sqrt{n}}\left(\skx\opnorm{C}^{\frac{1}{2}}(2+\sqrt{2}) + \nor{G_*}_{*}\skx\sqrt{2\risk(g_*)}\right)\sqrt{\ln\Big( \frac{28\mathsf{r}}{\delta}\Big)} + \sky\sqrt{\ln \Big( \frac{2}{\delta}\Big)\frac{1}{n}}
}
Rearranging we get
\eqal{\label{eq:eq_for_comp_hs}
\risk(\hat{g}) - \risk(g_*)\leq
\frac{\nor{G_*}_{*}}{\sqrt{n}}\Big[\Big((\sqrt{2}+1)\skx\tnorm{G_*}\nor{C}_{\rm op}^{\frac{1}{2}}+ \skx\sqrt{\risk(g_*)}\Big)\sqrt{2\ln \Big( \frac{28\mathsf{r}}{\delta}\Big)}\Big]+\sky\sqrt{\ln \Big( \frac{2}{\delta}\Big)\frac{1}{n}} 
}

with probability greater or equal to $1-3\delta$.  Bounding $\ln \left( \frac{2}{\delta}\right)$ with $\ln \left( \frac{28\mathsf{r}}{\delta}\right)$ we get 
\eqals{
\risk(\hat{g}) - \risk(g_*)\leq (\mathsf{c} +\sky)\sqrt{\frac{\ln(\frac{\mathsf{r}}{\delta})}{n}} + O(n^{-1})
}
where $\mathsf{c} = (2 + \sqrt{2})\skx\tnorm{G_*}^2\nor{C}_{\rm op}^{\frac{1}{2}}+\sqrt{2}\tnorm{G_*}\skx\risk(g_*)$. In the main body of the paper we bound it as 
\eqals{
\risk(\hat{g}) - \risk(g_*)\leq (\mathsf{c} +\sky)\sqrt{\frac{4\ln(\frac{\mathsf{r}}{\delta})}{n}} + O(n^{-1})
}
with $\mathsf{c} = 2\skx\tnorm{G_*}^2\nor{C}_{\rm op}^{\frac{1}{2}}+\tnorm{G_*}\skx\risk(g_*)$ to make it neater.
\end{proof}

\textnormal{\textbf{Comparison with Hilbert-Schmidt regularization.}} The goal of this remark is a comparison between the constants in the bound for the trace norm estimator and in the bound we would obtain with Hilbert-Schmidt estimator. 

\textbf{Bound for HS-regularization.} We show here the bound obtained with Hilbert-Schmidt regularization. In this case, $\GG_\gamma:=\{g(\cdot) = G\phi(\cdot) \mid \nor{G}_{\hs}\leq \gamma\}$.
Note that if $G$ is a Hilbert-Schmidt operator, then $G^* G$ is a trace norm operator. Therefore, the term $\scal{G^*G}{\hat C - C}_{\hs}$ can be bounded as before:

\eqals{
\opnorm{C-\hat{C}} \leq \frac{4}{n}\Big( \frac{\skx^2}{3}\ln\Big(\frac{14\bar{r}_1}{\delta}\Big)\Big) + \sqrt{\frac{4\skx^2\opnorm{C}}{n}\ln\Big(\frac{14 \bar{r}_1}{\delta}\Big)}
}

On the other hand, for  the second term $\scal{G}{\hat Z - Z}_{\hs}$  we have \eqals{
\left|\scal{G}{\hat Z - Z}_{\hs}\right|\leq \nor{G}_{\hs}\nor{\hat Z - Z}_{\hs}.
}
Now, in order to bound $\nor{\hat Z - Z}_{\hs}$, we note that $\nor{Z}_{\hs}^2 \leq \skx^2 \EE \nor{\psi(y)}^2_{\hs}=\skx^2\tr(C_Y)$. Proceeding in a similar way as in \cref{lem:boundnorm_cy}, we obtain that $\tr(C_Y)\leq \risk(g_*) +\skx^2\nor{G_*}_{\hs}$ and hence $\nor{Z}^2_{\hs}\leq \skx^2\risk(g_*) + \skx^4 \nor{G_*}_{\hs}^2.$
From Lemma 2 in \citep{Smale2007}, \eqals{
\nor{\hat Z - Z}_{\hs}\leq \sqrt{\frac{2(\skx^2\risk(g_*) + \skx^4 \nor{G_*}_{\hs}^2)}{n}}\sqrt{\ln\Big(\frac{2}{\delta}\Big)} + O(n^{-1}).
}
Finally, no difference holds for the last term $\left|\EE\nor{\psi(y)}_\hy^2-\frac{1}{n}\sum_{i=1}^n\nor{\psi(y_i)}_\hy^2\right|$. Hence, combining the three parts and bounding $\ln(\frac{2}{\delta})$ with  $\ln(\frac{14\mathsf{r}}{\delta})$, we get
\eqal{\label{eq:bound_HS}
\risk(\hat g_{\hs}) - \risk(g_*) \leq \frac{\nor{G_*}_{\hs}}{\sqrt{n}}\Big( \nor{G_*}_{\hs} \nor{C}_{\rm op}^{\frac{1}{2}}2\skx +  \sqrt{2}\skx^2 \nor{G_*}_{\hs}  + \skx\sqrt{2\risk(g_*)} + \sky \Big)\sqrt{\ln \Big(\frac{14\mathsf{r}}{\delta}\Big)}+ O(n^{-1}).
}
Note that this bound slightly refines the excess risk bounds for HS regularization provided in \citep{CilibertoRR16}.\\

\textbf{Comparison and discussion.} Let us compare the bound with HS regularization in \cref{eq:bound_HS} with the bound for the trace norm estimator that we derived in the proof of \cref{thm:gen_boun_g}:
\eqals{
\risk(\hat{g}) - \risk(g_*)\leq
\frac{\nor{G_*}_{*}}{\sqrt{n}}\Big(\tnorm{G_*}\nor{C}_{\rm op}^{\frac{1}{2}}\skx+ \skx\sqrt{2\risk(g_*)}+\sky\Big)\sqrt{\ln \Big( \frac{28\bar{r}}{\delta}\Big)} + O(n^{-1}).
}
To make the comparison easier, we isolate the constants in the bounds: \eqals{&\textnormal{HS: }
2\skx \nor{G_*}^2_{\hs}\nor{C}_{\rm op}^{\frac{1}{2}} + \underline{\sqrt{2}\skx\nor{G}_{\hs}^2}+\skx\risk(g_*)\nor{G_*}_{\hs}+\sky\quad \textnormal{versus}\\
&\textnormal{TN: } (2+\sqrt{2})\skx\tnorm{G_*}^2\nor{C}_{\rm op}^{\frac{1}{2}}+\skx\risk(g_*)\tnorm{G_*}+\sky.
}

We can summarize the cases as below.

\begin{itemize}
\item If $\nor{G_*}_{\hs} \ll\tnorm{G_*}$, then the TN bound  gives no advantage over the HS one.

\item Whenever $\nor{G_*}_{\hs}$ and $\tnorm{G_*}$ are of the same order,  our result shows an advantage in the constant of the bound: indeed, while in trace norm case, the norm $\tnorm{G_*}$ is mitigated by $\nor{C}_{\rm op}^{\frac{1}{2}}$, in the HS case is it not, because of the extra term $\nor{G_*}_{\hs}^2\skx$.  Note that $\nor{C}_{\rm op}\leq \skx$ and the gap between the two can be significant: for instance, if $C$ is the covariance operator of a uniform distribution on a $d$-dimensional unit sphere, $\nor{C}_{\rm op}=1/d$ while $\skx=1$.
Hence the  entity of the improvement depends on how smaller $\nor{C}_{\rm op}$ is with respect to $\tr(C)$.

The point above holds true when the other quantities $(\sky,\tnorm{G}\risk(G_*))$ in the constant  do not dominate.  However, this is reasonable to expect:\\
-$\risk(g_*)$ is the minimum expected risk; \\
-$\skx$ is $1$ whenever we choose a normalized kernel on the input (Gaussian);\\
-$\sky$ is also typically $1$: $\sky$ is such that $\sup_{y\in\yy}\nor{k_{\suby}(y,\cdot)}_\hy\leq \sky^2$ where $h$ is a reproducing kernel on the output. Whenever $\yy$ is finite (and hence $k_{\suby}(y,y')=\delta_{y==y'}$) or the loss is smooth (and hence $h$ is the Abel kernel), $\sky=1$. 
\end{itemize}



\section{Theoretical Analysis: Multitask Case}\label{sec:appe_mtl}
 We consider the general multitask learning case which allows a different loss function for each task: the goal is to minimize the \textit{multi-task excess risk} \eqals{\min_{f:\xx\rightarrow\conset}\ee(f),\quad\ee(f)=\frac{1}{T}\sum_{t=1}^T\int_{\xx\times\R}\ell_t(f_t(x),y)d\rho_t(x,y),}
where the $\rho_t$ are unknown probability distributions on $\xx\times \R$ that are observed via finite samples ${(x_{it},y_{it})}_{i=1}^{n_t}$, for $t=1,\dots,T$ .
Each $\ell_t$ is required to satisfy the SELF assumption in \cref{def:self}, i.e. \eqals{
\ell_t(y,y') = \sp{\psi_t(y)}{V_t \psi_t(y')},
}
and for $t=1,\dots T$ $\skyt$ is a constant such that $\sup_{y\in\yy} \nor{\psi_t(y)}\leq \skyt.$
In this setting the surrogate problem corresponds to 
\eqals{
\min_{G:\hx\rightarrow\hy^T}\risk(G)\qquad \risk_T(G):=\frac{1}{T}\int_{\xx\times \yy} \nor{\psi_t(y) - G_t \phi(x)}_{\hy}^2d\rho_t(x,y),
}
and its solution is denoted with $G_*$. 
Note that each $G_t$ is an operator in $\hx \otimes \hy$ and $G$ denotes the operator from $\hx$ to $\hy^T$ whose $t^{th}$ component is $G_t$, $t=1,\dots T$.
Formally, $G = \sum_{t=1}^T G_t\otimes e_t$, with $(e_t)_{t=1}^T$ the canonical basis of $\R^T$. Since $\nor{G}_{\hs}^2 = \sum_{t}\nor{G_t}_{\hs}^2$, in case of HS regularization the surrogate problem considers each task $t$ separately.\\

Here we perform regularization with  trace norm of the operator $G$. Setting $\GG_\gamma=\{ g(\cdot)=G\phi(\cdot)\mid G:\,\, \hx\rightarrow \hy^T \textnormal{ is s.t. } \tnorm{G}\leq \gamma \}$, we study the estimator $\hat g$ given by
\eqal{\label{eq:surr_mtl_tn}
\hat{g} = \argmin_{g\in\GG_\gamma} \frac{1}{T}\sum_{t=1}^T \frac{1}{n_t}\sum_{i=1}^n \nor{g_t(x_{it}) - \psi_t(y_{it})}_{\hy}^2.
}

In the following we will consider $n_t = n$ for simplicity and denote $\risk_T$ with $\risk$, to avoid cumbersome notation. The estimator $\hat g$ satisfies the following excess risk bound:
\begin{theorem}\label{thm:complete_thm_mtl_g}
For $t=1,\dots T$, let $(x_{it},y_{it})_{i=1}^n$ be an  iid sample of $\rho_t$ and $\hat{g}$ is the solution of \cref{eq:surr_mtl_tn} with $\gamma = \tnorm{G_*}$. 
\eqals{\risk(\hat g)- \risk(g_*) \leq \frac{1}{\sqrt{nT}}\Big(\tnorm{G_*}^2\nor{\bar{C}}_{\rm op}^{\frac{1}{2}}\skx(2+\sqrt{2})+\skx\tnorm{G_*}\sqrt{2\risk(G_*)} + \bar{\sky}\Big)\sqrt{\ln\Big(\frac{T \mathsf{r}}{\delta}\Big)} + O((nT)^{-1}),
} 
with probability greater or equal then $1- 3\delta$, where $\bar{C}$ is the average covariance operator, $\risk(G_*)$ the  expected true risk, $\bar{\sky} = \sqrt{\frac{1}{T}\sum_t\skyt^2}$ and $\mathsf{r}$ a number independent of $n,T,\delta$ and $G_*$.
\end{theorem}

The section is devoted to the proof of this result, which is the formal version of theorem \cref{thm:bounds_g_mlt} in the main paper.
\noindent We split the error as follows:
 \eqals{\risk(\hat{g}) - \risk(g_*)\leq& \risk(\hat{g}) - \hat{\risk}(\hat{g})+
 \hat{\risk}(\hat{g}) - \hat{\risk}(g_{\gamma *})\\
 +&\hat{\risk}(g_{\gamma *}) -  \risk(g_{\gamma *})+ 
 \risk(g_{\gamma *}) - \risk(g_*).}
 
Now, by definition of $\hat{g}$ the term $\hat{\risk}(\hat{g}) - \hat{\risk}(g_{\gamma *})$ is negative. Also, denoting by $\rho_{t\mid\xx}$ the marginal on $\xx$ of the probability measure $\rho_t$,
\eqals{\risk(g_{\gamma *}) - \risk(g_*)& = \frac{1}{T}\sum_{t=1}^T\int_{\xx}\nor{G_{t\gamma *}\phi(x) - G_{t*}\phi(x)}_{\hy}^2\,d\rho_{t\mid\xx}(x) \\
& = \inf_{\{ G \in\GG_\gamma \}}\frac{1}{T}\sum_{t=1}^T\nor{G_t\phi(x) - G_{t*}\phi(x)}^2_{L^2(\rho_{t\mid\xx})}\\
&\leq \frac{1}{T}\sum_{t=1}^T\nor{\big(\frac{\gamma}{\tnorm{G_{*}}}\big)G_{t*}\phi(x) - G_{t*}\phi(x)}^2_{L^2(\rho_{t\mid\xx})} \\
& \leq \Big( 1-\frac{\gamma}{\tnorm{G_*}}\Big)^2 \frac{1}{T}\sum_{t=1}^T\nor{G_{t*}\phi}_{L^2(\rho_{t\mid\xx})}^2\\
&\leq \Big( 1-\frac{\gamma}{\tnorm{G_*}}\Big)^2 \frac{\skx^2}{T}\nor{G_*}_{\hs}^2\leq (\tnorm{G_*} - \gamma)^2\frac{\skx^2}{T}. }

\noindent It remains to bound $R_1:=\risk(\hat{g}) - \hat{\risk}(\hat{g})$ and $R_2:=\hat{\risk}(g_{\gamma *}) -  \risk(g_{\gamma *})$. Since
\eqals{R_1+R_2 \leq 2\sup_{G\in\GG_\gamma}\abs{\hat{\risk}(G) - \risk(G)},}
we just have to bound the term on the right hand side.
In the following we assume $n_t = n$ for $t=1,\dots,T$ for clarity. Also, the notation $\EE u(x_t)\otimes v(y_t)$ is to be interpreted as $\EE_{(x,t)\sim \rho_t}u(x)\otimes v(y)$. 
For $t=1,\dots,T$, denote
\eqals{
    & C_t = \EE \phi(x_t)\otimes\phi(x_t), \qquad \hat C_t = \frac{1}{n} \sum_{i=1}^n \phi(x_{it})\otimes\phi(x_{it}) \\
    & Z_t = \EE \psi(y_t)\otimes\phi(x_t), \qquad \hat Z_t = \frac{1}{n} \sum_{i=1}^n \psi(y_{it})\otimes\phi(x_{it}). 
}
For any operator $G$, we have
\eqal{\label{eq:mtl_boun_1}
    \abs{\hat{\risk}(G) - &\risk(G)} = \left| \frac{1}{T}\sum_{t=1}^n\frac{1}{n}\sum_{i=1}^n \nor{\psi(y_{it})-G_t\phi(x_{it})}_\hy^2 - \EE \nor{\psi(y_t) - G_t\phi(x_t)}_\hy^2 \right| \\
    & = \Bigg| \frac{1}{T} \sum_{t=1}^T\frac{1}{n} \sum_{i=1}^n\left( \scal{G_t^*G_t}{\phi(x_{it})\otimes\phi(x_{it})}_{\hs}  - 2\scal{G_t}{\psi(y_{it})\otimes\phi(x_{it})}_{\hs} + \nor{\psi(y_{it})}_\hy^2 \right) \\
    & \qquad\qquad - \EE \left(\scal{G_t^*G_t}{\phi(x_t)\otimes\phi(x_t)}_\hs - 2 \scal{G_t}{\psi(y_t)\otimes\phi(x_t)}_\hs + \nor{\psi(y_t)}_\hy^2 \right) \Bigg| \\
    &\label{eq:split_err} = \left| \frac{1}{T}\sum_{t=1}^T\scal{G_t^*G_t}{\hat C_t - C_t}_{\hs} - 2\scal{G_t}{\hat Z_t - Z_t}_{\hs} + \frac{1}{T}\sum_{t=1}^T\frac{1}{n}\sum_{i=1}^n \nor{\psi(y_{it})}_\hy^2 - \EE \nor{\psi(y_t)}_\hy^2 \right| 
}
We analyse each term separately in the following lemmas.

\begin{lemma}\label{lem:bound_term_1}
The first term in \cref{eq:split_err} satisfies the following inequality:
\eqals{\left |\frac{1}{T}\sum_{t=1}^T\scal{G_t^*G_t}{\hat C_t - C_t}_{\hs}\right | \leq \frac{1}{T}\frac{4\nor{G}_{\hs}^2}{n} \Big(\frac{\skx^2}{3}\ln\Big( \frac{T\mathsf{r}_1}{\delta}\Big)\Big) + \frac{\nor{G}_{\hs}^2}{T} \sqrt{\frac{4 \skx^2 \max_{t}\nor{C_t}_{\rm op}}{n}\ln\Big(\frac{T\mathsf{r}_1}{\delta}\Big)}\Big)}
with probability $1-\delta$, where $\mathsf{r}_1$ is a constant independent of $n,T,G$ and which is given by the problem.
\end{lemma}
\begin{proof}
\eqals{
\frac{1}{T}\sum_{t=1}^T\scal{G_t^*G_t}{\hat C_t - C_t}_{\hs} = \frac{1}{T} \textnormal{tr}(\mathbf{G^*}\mathbf{C})\leq \frac{1}{T}\tnorm{\mathbf{G}} \nor{\mathbf{C}_{\rm op}},
}
where $\mathbf{G} = \sum_{t=1}^T (e_t\otimes e_t) \otimes (G_t^* G_t)$ and $\mathbf{C} = \sum_{t=1}^T (e_t\otimes e_t)\otimes (\hat C_t - C_t)$.
Now, \eqals{\nor{\mathbf{G}}_{*} = \sum_{t=1}^T \nor{G_t^*G_t}_* = \sum_{t=1}^T \nor{G_t G_t^*}_{\hs} = \nor{G}_{\hs}^2}
and  \eqals{
\nor{\mathbf{C}}_{\rm op} = \max_{t=1,\dots,T} \nor{C_t - \hat C_t}_{\rm op}.
}
Using \cref{lem:berst_ineq}, we get 
\eqals{\nor{C_t - \hat C_t}_{\rm op} \leq \frac{4}{n} \Big(\frac{\skx^2}{3}\ln\Big( \frac{14\bar{r}_t}{\delta}\Big)\Big) +  \sqrt{\frac{4 \skx^2 \nor{C_t}_{\rm op}}{n}\ln\Big(\frac{14\bar{r}_t}{\delta}\Big)}\Big)}
with probability greater than $1-\delta$. Performing an intersection bound we have that for $t=1,\dots,T$ \eqals{
\max_{t=1,\dots,T} \nor{C_t - \hat C_t}_{\rm op} \leq  \frac{4}{n} \Big(\frac{\skx^2}{3}\ln\Big( \frac{\mathsf{r}_1}{\delta}\Big)\Big) +  \sqrt{\frac{4 \skx^2 \nor{C_t}_{\rm op}}{n}\ln\Big(\frac{\mathsf{r}_1}{\delta}\Big)}\Big)\\
 =  \frac{4}{n} \Big(\frac{\skx^2}{3}\ln\Big( \frac{\mathsf{r}_1}{\delta}\Big)\Big) +  \sqrt{\frac{4 \skx^2 \max_{t}\nor{C_t}_{\rm op}}{n}\ln\Big(\frac{\mathsf{r}_1}{\delta}\Big)}\Big)}
with probability $1-T\delta$, where $\mathsf{r}_1 = 14\max_t\bar{r}_t$. With some abuse of notation take $\delta = \delta/T$ and we get 
\eqals{
\max_{t=1,\dots,T} \nor{C_t - \hat C_t}_{\rm op} \leq    \frac{4}{n} \Big(\frac{\skx^2}{3}\ln\Big( \frac{T\mathsf{r}_1}{\delta}\Big)\Big) +  \sqrt{\frac{4 \skx^2 \max_{t}\nor{C_t}_{\rm op}}{n}\ln\Big(\frac{T\mathsf{r}_1}{\delta}\Big)}\Big)}
with probability $1-\delta$.
\end{proof}

\begin{lemma}\label{lem:bound_term_2}
The following bounds holds true:
\eqals{\frac{1}{T}\sum_{t=1}^T \scal{G_t}{\hat Z_t - Z_t}\leq \frac{4\nor{G}_*}{nT}\Big( \frac{\mathsf{k}_1 }{3}\ln\Big(\frac{28\bar{r}}{\delta}\Big)\Big) +
 \frac{\skx\nor{G}_*}{T\sqrt{n}}\left(\sqrt{2T\risk(G_*)} + \skx \nor{G} \sqrt{\nor{\sum_t C_t}_{\rm op}}\right)\ln\Big(\frac{\mathsf{r}_2}{\delta}\Big)}
with probability at least $1-\delta$, 
where $\mathsf{k}_1 = \skx\max_t\skyt$,  and $\mathsf{r}_2$ is independent of $G_*$, $\delta$, $n$ and $T$.
\end{lemma}
\begin{proof}
Let us start with the following bound
 \eqals{
\frac{1}{T}\sum_{t=1}^T \scal{G_t}{\hat Z_t - Z_t} = \frac{1}{T} \textnormal{tr}(G \mathbf{Z})\leq \frac{1}{T}\nor{G}_* \nor{\mathbf{Z}}_{\rm op}
}
where $\mathbf{Z}= \sum_{t=1}^T (\hat Z_t - Z_t) \otimes e_t$. To bound $\nor{\mathbf{Z}}_{\rm op}$ some extra work is needed. We aim to apply \cref{lem:berst_ineq_2} again. Let us define \eqals{
X_{it} = \frac{1}{n}\big(\psi(y_{it}) \otimes \phi(x_{it}) - \EE \psi(y_t) \otimes \phi(x_t) \big)\otimes e_t, 
}
so that $\sum_{i,t} X_{it} = \mathbf{Z}.$ 
Note that \eqals{
\nor{X_{it}}_{\rm op}\leq \frac{\max_t\skx\skyt + \max_t\nor{Z_t}_{\rm op}}{n}\leq \frac{2\max_t\skx\skyt }{n}\quad \textnormal{for any }i,t.
}
In order to apply \cref{lem:berst_ineq_2} we need to bound the following two quantities \eqals{
\nor{\sum_{i,t} \EE X_{it} X_{it}^*}, \qquad \nor{\sum_{i,t} \EE X_{it}^* X_{it}}.
}
Note that \eqals{
&\EE X_{it} X_{it}^* =\frac{1}{n^2} \big(\EE (\psi(y_{it})^2 \phi(x_{it})\otimes\phi(x_{it})) - Z_t Z_t^*\big)\\
&\sum_{it}\EE X_{it} X_{it}^*=  \frac{1}{n}\sum_t(\frac{1}{n}\sum_{i=1}^n\EE(\psi(y_{it})^2 \phi(x_{it})\otimes\phi(x_{it}))-Z_t Z_t^*)}
and hence
\eqal{\label{eq:part1misker4z}\nor{\sum_{it}\EE X_{it} X_{it}^*}_{\rm op}\leq \frac{2}{n}\opnorm{\sum_t\EE (\psi(y_{t})^2 \phi(x_{t})\otimes \phi(x_t))}.}
A direct application of  \cref{lem:bound_ky_mlt} yields 
\eqals{
\nor{\sum_{it}\EE X_{it} X_{it}^*}_{\rm op}\leq \frac{2\skx^2}{n}(T\risk(g_*) + \opnorm{\sum_t C_t}\nor{G}_{\hs}^2).
}

Also, 
\eqals{&X_{it}^* X_{it} = \frac{1}{n^2}(k(x_{it},x_{it})\psi(y_{it})^2 - \nor{\EE(\psi(y_{t})\otimes \phi(x_t)}^2)e_t \otimes e_t}
and \eqals{
\sum_{it}\EE X_{it}^* X_{it} &= \frac{1}{n}\sum_t \frac{1}{n} \sum_{i=1}^n(\EE
k(x_{it},x_{it})\psi(y_{it})^2 -\nor{\EE(\psi(y_{t})\otimes \phi(x_t)}^2  )e_t\otimes e_t\\
&\preceq  
\frac{1}{n} \sum_t (\skx^2 C_{Y,t} -\nor{\EE(\psi(y_{t})\otimes \phi(x_t)}^2  )e_t\otimes e_t
.}
Taking the operator norm, we obtain 

\eqal{\label{eq:part2misker4z}
\nor{\sum_{it}\EE X_{it}^* X_{it} }_{\rm op} \leq \frac{2}{n}\max_{t=1,\dots,T}(\skx^2 \opnorm{C_{Y,t}})\leq \frac{2\skx^2}{n} \opnorm{\sum_t C_{Y,t}}\leq \frac{2\skx^2}{n}(T\risk(G_{*})+\nor{G_*}_{\hs}^2\nor{\sum_t C_t}_{\rm op}).
}

 where the last inequality follows by \cref{lem:boundnorm_cy_mtl}.

Both $\nor{\sum_{it}\EE X_{it}^* X_{it} }_{\rm op}$ and $\nor{\sum_{it}\EE X_{it} X_{it}^* }_{\rm op} $ are upper bounded by $\frac{2\skx^2}{n}(T\risk(G_{*})+\nor{G_*}_{\hs}^2\nor{\sum_t C_t}_{\rm op})$.

Then, by \cref{lem:berst_ineq_2}, we have \eqals{
\nor{\mathbf{Z}}_{\rm op} \leq \frac{4}{n}\Big( \frac{\max_t\skx\skyt}{3}\ln\Big(\frac{28\bar{r}}{\delta}\Big)\Big)+ \sqrt{\frac{2\skx^2(T\risk(G_{*})+\nor{G_*}_{\hs}^2\nor{\sum_t C_t}_{\rm op})}{n}\ln\Big(\frac{28 \bar{r}}{\delta}\Big)},}
where $\bar{r}$ is the effective rank of $\mathbf{Z}$.
Rearranging we get 
\eqals{\nor{\mathbf{Z}}_{\rm op}\leq \frac{4}{n}\Big( \frac{\mathsf{k}_1 }{3}\ln\Big(\frac{28\bar{r}}{\delta}\Big)\Big) + \skx\sqrt{\frac{
2T\risk(G_*)}{n}\ln\Big(\frac{\mathsf{r}_2}{\delta}\Big)}
+ \tnorm{G}\skx\sqrt{\nor{\sum_t C_t}_{\rm op}\frac{2}{n}\ln\Big(\frac{\mathsf{r}_2}{\delta}\Big)} 
}
where $\mathsf{k}_1=\skx\max_t\skyt$ and $\mathsf{r}_2 = 28\bar{r}$.
\end{proof}

\begin{lemma}\label{lem:final_part}
Recall that $\nor{\psi(y_{it})}^2\leq \skyt^2$ for $i=1,\dots,n$, $t=1,\dots T$.
\eqals{\left| \frac{1}{T}\sum_{t=1}^T\frac{1}{n}\sum_{i=1}^n \nor{\psi(y_{it})}_\hy^2 - \EE \nor{\psi(y_t)}_\hy^2 \right|\leq \sqrt{\Big(\frac{1}{T}\sum_t\skyt^2\Big)\frac{1}{nT}\ln\Big(\frac{2}{\delta}\Big)}}
with probability at least $1-\delta$. 

\end{lemma}
\begin{proof}
The bound follows by a direct application of Hoeffding inequality.
\end{proof}

We are now ready to prove theorem \cref{thm:complete_thm_mtl_g}.
\begin{proof}
Recall that 
\eqals{\risk(\hat g) -\risk(g_*) \leq (\nor{G_*}_* - \gamma)^2\frac{\skx^2}{T} + 2\sup_{G\in\GG_\gamma}\abs{\risk(\hat{G})- \risk(G))}.
}
 Recall that for any $G\in\GG_\gamma$, $\tnorm{G}\leq \gamma$. Now, combining \cref{eq:split_err} and  \cref{lem:bound_term_1}, \cref{lem:bound_term_2} and \ref{lem:final_part}, we get
  \eqals{
\risk(\hat g) - \risk(g_*)\leq ((\nor{G_*}_* - \gamma)^2\frac{\skx^2}{T} + \gamma^2\left(\frac{A}{n}+\frac{B}{\sqrt{n}}\right) + \gamma\left(\frac{A'}{n}+\frac{B'}{\sqrt{n}}\right) + 
\sqrt{\Big(\frac{1}{T}\sum_t\skyt^2\Big)\frac{1}{nT}\ln\Big(\frac{2}{\delta}\Big)},
} 
where
 \eqals{
&A =  \frac{4}{T} \Big(\frac{\skx^2}{3}\ln\Big( \frac{T\mathsf{r}_1}{\delta}\Big)\Big) \qquad B = \frac{1}{T} \sqrt{4 \skx^2 \max_{t}\nor{C_t}_{\rm op}\ln\Big(\frac{T\mathsf{r}_1}{\delta}\Big)}\Big) +\frac{1}{T}\skx\sqrt{2\nor{\sum_t C_t}_{\rm op}\ln\Big(\frac{\mathsf{r}_2}{\delta}\Big)}  \\
&A'= \frac{4}{T}\Big( \frac{\max_t\skx\skyt }{3}\ln\Big(\frac{\mathsf{r}_2}{\delta}\Big)\Big)\qquad B'=
\frac{1}{T}\sqrt{2T \risk(G_{*})\ln\Big(\frac{\mathsf{r}_2}{\delta}\Big).}
}

Optimizing with respect to $\gamma$ we could find the optimal parameter and compute the corresponding bound. However, in the following we choose $\gamma = \tnorm{G_*}$,  so that the approximation error is zero. 
In the following we will bound $\max_t \opnorm{C_t}$ with $\opnorm{\sum_t C_t}$ and both the logarithm terms with $\ln (\frac{T \mathsf{r}}{\delta})$ for a suitable $\mathsf{r}$ (e.g. $\max(\mathsf{r}_1,\mathsf{r}_2)$). Isolating the faster term we obtain

 \eqals{
\risk(\hat g)- \risk(g_*) \leq &\tnorm{G_*}^2\skx\left[\frac{\sqrt{\nor{\sum_t C_t}_{\rm op}}}{T}\frac{2+\sqrt{2}}{\sqrt{n}}+\frac{\skx\tnorm{G_*}}{T}\sqrt{\frac{2T \risk(G_{*})}{n}}\right]\ln\Big(\frac{T\mathsf{r}}{\delta}\Big)\\
&+ \sqrt{\Big(\frac{1}{T}\sum_t\skyt^2\Big)\frac{1}{nT}\ln\Big(\frac{2}{\delta}\Big)}+O((nT)^{-1}).
} 
Denote  by $\bar{C}$ the average of $C_1, \dots, C_T$, i.e. $\frac{1}{T}\sum_t C_t$. Then,
 \eqals{
\frac{1}{T}\sqrt{\nor{\sum_t C_t}_{\rm op}} = \frac{1}{\sqrt{T}}\nor{\bar{C}}_{\rm op}^{\frac{1}{2}}.
}
Rearranging the terms we get the final bound  
\eqals{\risk(\hat g)- \risk(g_*) \leq \frac{1}{\sqrt{nT}}\Big(\tnorm{G_*}^2\nor{\bar{C}}_{\rm op}^{\frac{1}{2}}\skx(2+\sqrt{2})+\skx\tnorm{G_*}\sqrt{2\risk(G_*)} + \bar{\sky}\Big)\sqrt{\ln\Big(\frac{T \mathsf{r}}{\delta}\Big)} + O((nT)^{-1}),
}
where $\bar{\sky} = \sqrt{\frac{1}{T}\sum_t\skyt^2}$.
 
\end{proof}
\vspace{1cm}

\section{Auxiliary Lemmas}\label{sec:auxiliary_lemma}

In this section we recall some auxiliary results that are used in the proofs of the work. 
Let us recall the definition of effective rank. Given an Hilbert space $\mathcal{H}$  let $A:\mathcal{H}\rightarrow \mathcal{H}$,  
 be a compact operator. The effective rank of $A$ is refined as \eqals{
 r(A) = \frac{\tr A}{\nor{A}_{\rm op}}.
 }
 
 \begin{lemma}\label{lem:berst_ineq}
Let $X_1,\dots, X_n \in \mathbb{C}^{d\times d}$ a sequence of independent self adjoint random matrices such that $\EE X_i=0$, for $i=1,\dots,n$ and $\sigma^2\geq \nor{\sum_{i=1}^n\EE X_i^2}_{\rm op}$. Assume that $\nor{X_i}\leq U$  almost surely for all $1\leq i\leq n$ and some positive $U \in \mathbb{R}$. Then, for any $t\geq \frac{1}{6} (U + \sqrt{U + 36\sigma^2})$, \eqal{
\mathbb{P} \Big(\nor{\sum_{i=1}^n X_i}_{\rm op} > t \Big) \leq 14 r(\sum_{i=1}^n \EE X_i^2) \exp\big(-\frac{t^2/2}{\sigma^2 + tU/3}\big),}
where $r(\cdot)$ denotes the effective rank. 

\end{lemma}

A similar results holds true for general matrices with no requirements on self adjointness:
\begin{lemma}\label{lem:berst_ineq_2}
Let $X_1,\dots, X_n \in \mathbb{C}^{d\times d}$ a sequence of independent  random matrices such that $\EE X_i=0$, for $i=1,\dots,n$ and $\sigma^2\geq max(\nor{\sum_{i=1}^n\EE X_i X_i^*}_{\rm op},\nor{\sum_{i=1}^n\EE X_i^* X_i}_{\rm op}$. Assume that $\nor{X_i}\leq U$  almost surely for all $1\leq i\leq n$ and some positive $U \in \mathbb{R}$. Then, for any $t\geq \frac{1}{6} (U + \sqrt{U + 36\sigma^2})$, \eqal{
\mathbb{P} \Big(\nor{\sum_{i=1}^n X_i}_{\rm op} > t \Big) \leq 28 \tilde{d} \exp\big(-\frac{t^2/2}{\sigma^2 + tU/3}\big),}
where $\tilde{d}=\max(r(\sum_{i=1}^n\EE X_i X_i^*),r(\sum_{i=1}^n\EE X_i^* X_i)))$ and $r(\cdot)$ denotes the effective rank.
\end{lemma}

The lemma above holds true for Hilbert Schmidt operators between separable Hilbert spaces, as shown in section $3.2$ in \citep{minsker2017some}.

\begin{lemma}\label{lem:boundnorm_cy}
The following bound on the operator norm of the covariance operator on the output $\EE \psi(y) \otimes \psi(y)$ holds true:
\eqals{
\opnorm{\EE \psi(y) \otimes \psi(y)}\leq \nor{G_*}_{\hs}^2 \nor{C}_{\rm op} + \risk(g_*).
}
\end{lemma}
\begin{proof}
Let us start for the identity below: \eqal{\label{eq:split_psipsi}
\psi(y)\otimes \psi(y) = (\psi(y) - G_*\phi(x)) \otimes  (\psi(y) - G_*\phi(x)) + G_*\phi(x) \otimes (\psi(y) - G_*\phi(x)) + \psi(y) \otimes G_*\phi(x).
}
Taking the expectation on the right hand side we obtain \eqals{
\EE ( (\psi(y) - G_*\phi(x)) \otimes  (\psi(y) - G_*\phi(x))) + \EE G_*\phi(x) \otimes (\psi(y) - G_*\phi(x)) + \EE \psi(y) \otimes G_*\phi(x).
}
Note that the second term is zero, since \eqals{
\EE G_*\phi(x) \otimes (\psi(y) - G_*\phi(x)) = &\int_{\xx\times \yy}  G_*\phi(x) \otimes (\psi(y) - G_*\phi(x))d\rho(x,y)  \\
=&\int_\xx G_* \phi(x) \Big(\int_\yy \psi(y)d\rho(y\mid x) - G_*\phi(x)  \Big)d\rho_\xx
} and $G_*\phi(x) = \int_\yy\phi(y)d\rho(y\mid x).$
As for the last term, we have \eqals{
\EE \psi(y)\otimes G_*\phi(x) = \int_\xx\int_{\yy} \psi(y)d\rho(y\mid x)\otimes G_*\phi(x)d\rho_\xx  =  \int_\xx G_*\phi(x)\otimes G_*\phi(x).
}
Taking the operator norm we get \eqals{
\opnorm{\EE \psi(y) \otimes \psi(y)}&\leq \nor{\EE ( (\psi(y) - G_*\phi(x)) \otimes  (\psi(y) - G_*\phi(x)))}_{\rm op} + \nor{G_* C G_*^*}_{\hs}\\
&\leq \risk(g_*) + \nor{G_* C G_*^*}_{\hs}^2\leq\risk(g_*) + \nor{G}_{\hs}^2\nor{C}_{\rm op}.
}
\end{proof}

\begin{lemma}\label{lem:bound_ky}
The following bound holds true 
\eqals{
\nor{\EE (\sp{\psi(y)}{\psi(y)}\phi(x)\otimes \phi(x))}_{\rm op}\leq \skx^2(\nor{G_*}_{\hs}^2 \nor{C}_{\rm op} + \risk(g_*)). 
}
\end{lemma}
\begin{proof}
 Let us rewrite $\EE (\sp{\psi(y)}{\psi(y)}\phi(x)\otimes \phi(x))$ as follows
 \eqal{\label{eq:first_eq}
 \int_{\xx\times \yy} \sp{\psi(y)}{\psi(y)}\phi(x)\otimes \phi(x)d\rho(x,y) = \int_{\xx}\phi(x)\otimes \phi(x) \left(\int_{\yy} \sp{\psi(y)}{\psi(y)}d\rho(y\mid x) \right)d\rho_\xx(x).
 }
 The inner integral corresponds to $\EE_{y\mid x} \tr (\psi(y)\otimes \psi(y)) = \tr ~\EE_{y\mid x} (\psi(y)\otimes \psi(y)).$ Writing $\psi(y)\otimes \psi(y)$ as in \cref{eq:split_psipsi} and integrating wrt $\rho(\cdot \mid x)$, we observe that 
 \eqals{
 \int_{\yy}\psi(y)\otimes \psi(y) d\rho(y\mid x) &= \int_{\yy} (\psi(y)- G_* \phi(x))\otimes (\psi(y)- G_* \phi(x))d\rho(y\mid x) \\
 &+ G_*\phi(x)\otimes \left(\int_{\yy} \psi(y)d\rho(y\mid x) - G_*\phi(x)\right) \\
 &+G_* \phi(x)\otimes \phi(x) G_*^*.
 }
 Since $\int_\yy \psi(y)d\rho(y\mid x) = G_*\phi(x)$, the second term on the right hand side is zero and hence 
 \eqals{\tr~ \EE_{y\mid x} \psi(y)\otimes \psi(y) = \tr~ \EE_{y\mid x}(\psi(y)- G_* \phi(x))\otimes (\psi(y)- G_* \phi(x))+ \tr~ G_* \phi(x)\otimes \phi(x) G_*^* .
 }
 Substituting it on the right hand side of \cref{eq:first_eq} and taking the operator norm and using the triangle inequality, we obtain 
 \eqals{
\opnorm{\int_{\xx}\phi(x)\otimes \phi(x)\left(\int_{\yy}\nor{\psi(y) - G_*\phi(x)}_{\hy}^2d\rho(y\mid x)\right)d\rho_\xx(x)} \\
\leq \skx^2 \int_{\xx\times \yy}\nor{\psi(y) - G_*\phi(x)}_{\hy}^2d\rho(y,x) = \skx^2\risk(g_*)
 }
 and 
 \eqals{
 \nor{\int_{\xx}\phi(x)\otimes \phi(x) ~\tr(G_* \phi(x)\otimes \phi(x) G_*^*)}_{\rm op}\leq \nor{C}_{\rm op}\skx^2 \nor{G_*}_{\hs}^2.
 }
 Combining the parts together leads to the desired inequality 
 \eqals{
 \opnorm{\EE (\sp{\psi(y)}{\psi(y)}\phi(x)\otimes \phi(x))}\leq \skx^2(\nor{G_*}_{\hs}^2 \nor{C}_{\rm op} + \risk(g_*)).
 }
\end{proof}

\begin{lemma}\label{lem:boundnorm_cy_mtl}
Let $C_{Y,t}$ denote the covariance on the output for the $t^{th}$ task, that is \eqal{
C_{Y,t} :=\EE \psi(y_t)\otimes \psi(y_t).
}
Then the following inequality holds true\eqal{
\opnorm{\sum_t C_{Y,t}}\leq \nor{G_{*}}_{\hs}^2 \nor{\sum_t C_t}_{\rm op} + T\risk(G_{*}).
}
\end{lemma}
\begin{proof}
Let us start from the identity below: \eqals{
\psi(y_t)\otimes \psi(y_t) & = (\psi(y_t) - G_{t*}\phi(x_t))\otimes(\psi(y_t) - G_{t*}\phi(x_t))  + G_{t*}\phi(x_t)\otimes(\psi(y_t) \\
&- G_{t*}\phi(x_t)) + \psi(y_t)\otimes  G_{t*}\phi(x_t).
}
Taking the expectation on the right hand side we obtain \eqals{
\EE ( (\psi(y_t) - G_{t*}\phi(x_t))\otimes  (\psi(y_t) - G_{t*}\phi(x_t) + \EE G_{t*}\phi(x_t) \otimes (\psi(y_t) - G_{t*}\phi(x_t)) + \EE \psi(y_t)\otimes G_{t*}\phi(x_t).
}
As in \cref{lem:boundnorm_cy}, note that the second term is zero. 
As for the last term, we have \eqals{
\EE \psi(y_t)\otimes G_{t*}\phi(x_t) &= \int_\xx\int_{\yy} \psi(y_t)d\rho_t(y\mid x)\otimes G_{t*}\phi(x)d\rho_{t\xx}  =  \int_\xx (G_{t*}\phi(x))\otimes  (G_{t*}\phi(x)) d\rho_{t,\xx} \\
&=  \int_\xx (G_{t*}\phi(x)\otimes \phi(x) G^*_{t*})\rho_{t,\xx} = G_{t*} C_t G_{t*}^* .
}
Therefore, summing on $t$ and taking the operator norm we get \eqals{
\opnorm{\sum_t C_{Y,t}}&\leq \nor{\sum_t\EE ( (\psi(y_t) - G_{t*}\phi(x_t))^2}_{\rm op} + \nor{ G_{t*} C_t G_{t*}^*}_{\hs}\\
&\leq \sum_t\risk(G_{t*}) + \nor{ \sum_t G_{t*} C_t G_{t*}^*}_{\hs} \leq  \sum_t\risk(G_{t*}) + \nor{ \sum_t G_{t*} \sum_s C_s G_{t*}^*}_{\hs} \\
&\leq\sum_t\risk(G_{t*}) + \sum_t\nor{G_{t*}}_{\hs}^2\nor{\sum_t C_t}_{\rm op} \leq T\risk(G_{*}) + \nor{G_{*}}_{\hs}^2\nor{\sum_t C_t}_{\rm op}.
}
\end{proof}

\begin{lemma}\label{lem:bound_ky_mlt}
The following bound holds true 
\eqals{
\nor{\sum_t\EE (\psi_t(y_t)^2\phi(x_t)\otimes \phi(x_t))}_{\rm op}\leq \skx^2(\nor{G_*}_{\hs}^2 \nor{\sum_t C_t}_{\rm op} + \risk(g_*)). 
}
\end{lemma}
\begin{proof}
 It is a immediate variation of the proof of \cref{lem:bound_ky}.
\end{proof}

\section{Equivalence between Tikhonov and Ivanov Problems for trace norm Regularization}\label{sec:relation-between-tikhonov-ivanov}

In this section we provide more details regarding the relation between the Tikhonov regularization problem considered in \cref{eq:trace norm-problem} and the corresponding Ivanov problem in \cref{eq:ivanov_tn}. As discussed in the paper this approach guarantees that theoretical results characterizing the excess risk of the Ivanov estimator extend automatically to the Tikhonov one. 

Let $(x_i,y_i)_{i=1}^n$ be a training set and consider $\Phi:\hx\to\R^n$ and $\Psi:\hy\to\R^n$ the operators 
\eqals{
    \Phi = \sum_{i=1}^n e_i \otimes \phi(x_i) \qquad {\textrm and} \qquad \Psi = \sum_{i=1}^n e_i \otimes \psi(y_i)
}
with $e_i\in\R^n$ the $i$-th element of the canonical basis in $\R^n$. We can write the empirical surrogate risk in compact operatorial notation as 
\eqals{
    \hat\risk(G) = \frac{1}{n} \sum_{i=1}^n \nor{G\phi(x_i)-\psi(y_i)}_\hy^2 =  \frac{1}{n} \nor{\Phi G^* - \Psi}_{\hy\otimes\R^n}^2.
}

\begin{proposition}[Representer Theorem for Trace Norm Regularization]\label{thm:representer-theorem-trace norm}
Let $\hat G\in\hx\otimes\hy$ be a minimizer of
\eqals{
    \min_{G\in\hx\otimes\hy}~ \hat\risk(G) + \lambda\tnorm{G}.
}
Then the range of $\hat G^*$ is contained in the range of $\Phi^*$, or equivalently
\eqals{
    \hat G(\Phi^\dagger\Phi) = \hat G,
}
where $\Phi^\dagger$ denotes the pseudoinverse of $\Phi$.
\end{proposition}

The proof of this result is essentially equivalent to the one in \citep[Thm. 3][]{Abernethy2008}. We report it here for completeness.

\begin{proof}
For any $G\in\hy\otimes\hx$, consider the factorization 
\eqals{
    G = G_0 + G_\perp \qquad {\textrm{with}} \qquad G_0 = G(\Phi^\dagger\Phi) \qquad {\textrm and} \qquad  G_\perp(I - (\Phi^\dagger\Phi)).
}
Note that $(\Phi^\dagger\Phi)\in\hx\otimes\hx$ corresponds to the orthogonal projector of $\hx$ onto the range of $\Phi^*$ in $\hx$ (equivalently onto the span of $(\phi(x_i))_{i=1}^n$). By construction, we have that $\Phi G^* = \Phi G_0^*$. Hence $\hat\risk(G) = \hat\risk(G_0)$. However, since $(\Phi^\dagger\Phi)$ is an orthogonal projector, we have that 
\eqals{
    \tnorm{G_0} = \tnorm{G(\Phi^\dagger\Phi)} \leq \tnorm{G},
}
with equality holding if and only if $G_0 = G$. 

Now, if $\hat G$ is a minimizer of the trace norm regularized ERM we have 
\eqals{
    \hat\risk(\hat G_0) + \lambda\tnorm{\hat G_0} & \geq \hat\risk(\hat G) + \lambda\tnorm{\hat G} \\
    & = \hat\risk(\hat G_0) + \lambda\tnorm{\hat G},
}
which implies $\tnorm{\hat G_0}\geq\tnorm{\hat G}$. 

We conclude that $\hat G = \hat G_0 = \hat G (\Phi^\dagger\Phi)$. This corresponds to the range of $G^*$ being contained in the range of $\Phi$ as desired.  
\end{proof}

\begin{proposition}\label{prop:trace norm-unique-minimizer}
The empirical risk minimization for $\hat\risk(G)+\lambda\tnorm{G}$ admits a unique minimizer. 
\end{proposition}

\begin{proof}
According to \cref{thm:representer-theorem-trace norm}, all minimizers of the trace norm regularized empirical risk minimization belong to the set 
\eqals{
    \mathcal{S} = \left\{G \in\hy\otimes\hx ~ \middle| ~ G(\Phi^\dagger\Phi) = G\right\}.
}
Hence we can restrict to the optimization problem
\eqals{
    \min_{G\in\mathcal{S}}~ \hat\risk(G) + \lambda\tnorm{G}.
}
Note that $\mathcal{S}$ is idetified by a linear relation and thus is a convex set and thus the problem above is a convex program. We now show that on $\mathcal{S}$ the ERM objective functional is actually strongly convex for the case of the least-squares loss. To see this, let us consider the Hessian of $\hat\risk(\cdot)$. We have that, the gradient corresponds to
\eqals{
    \nabla\hat\risk(G) = \frac{2}{n}\left(G \Phi^*\Phi - \Psi^*\Phi \right),
}
and therefore the Hessian is the operator $\nabla^2\hat\risk(G):\hy\otimes\hx\to\hy\otimes\hx$ such that 
\eqals{
    \nabla^2\hat\risk(G)H = \frac{2}{n} H \Phi^*\Phi,
}
for any $H\in\hy\otimes\hx$ \citep[see e.g.][]{kollo2006advanced}. Now, we have that for any $H\in\mathcal{S}$
\eqals{
    \scal{H}{\nabla^2\hat\risk(G) H}_{\hy\otimes\hx} = \frac{2}{n}\scal{H}{H\Phi^*\Phi}_{\hy\otimes\hx} = \frac{2}{n}\tr(H^*H\Phi^*\Phi).
}
Now, let $r\leq n$ be the rank of $\Phi$ and consider the singular value decomposition of $\Phi = U\Sigma V^*$, with $U\in\R^{n\times r}$ a matrix with orthonormal columns $V\in\hx\to\R^r$ a linear operator such that $V^*V = I\in\R^{r\times r}$ and $\Sigma\in\R^{r \times r}$ a diagonal matrix with {\em all positive diagonal elements}. Then, 
\eqals{
    \tr(H^*H\Phi^*\Phi) = \tr(H^*H V\Sigma^2V^*) = \tr(V^*H^*HV\Sigma^2) \geq \sigma_{\min}^2 \nor{HV}_{\hy\otimes\R^r}^2,
}
where $\sigma_{\textrm min}^2$ denotes the smallest singular value of $\Sigma$ (equivalently, $\sigma_{\textrm min}$ is the smallest singular value of $\Phi$ {\em greater than zero}). 

Now, recall that $H\in\mathcal{S}$. Therefore
\eqals{
    H = H(\Phi^\dagger\Phi) = H VV^*,
}
which implies that 
\eqals{
    \nor{H}_{\hy\otimes\hx}^2 = \tr(H^*H) = \tr(VV^*H^*HVV^*) = \tr(V^*H^*HV(V^*V)) = \tr(V^*H^*HV) = \nor{HV}_{\hy\otimes\R^r}^2,
}
where we have used the orthonormality $V^*V = I\in\R^{r\times r}$. 

We conclude that
\eqals{
    \scal{H}{\nabla^2\hat\risk(G) H}_{\hy\otimes\hx} \geq \frac{2\sigma_{\textrm min}^2}{n} \nor{H}_{\hy\otimes\hx}^2,
}
for any $H\in\mathcal{S}$. Note that $\sigma_{\textrm min}>0$ is greater than zero since it is the smallest singular value of $\Phi$ greater than zero and $\Phi$ has finite rank $r\leq n$. Hence, on $\mathcal{S}$, the function $\hat\risk(G)$ is strongly convex. As a consequence also the objective functional $\hat\risk(G) + \lambda\tnorm{G}$ is strongly convex and thus admits a unique minimizer, as desired. 
\end{proof}

We conclude this section by reporting the result stating the equivalence between Ivanov and Tikhonov for trace norm regularization. 

In the following we will denote by $G_\lambda$ the minimizer of the Tikhonov regularization problem corresponding to minimizing $\hat\risk(G)+\lambda\tnorm{G}$ and by $G^I_\gamma$ the minimizer of the Ivanov regularization problem introduced in \cref{eq:ivanov_tn}, namely
\eqals{
    \min_{\tnorm{G}\leq\gamma} ~ \hat\risk(G).
}
We have the following.

\begin{theorem}\label{thm:equivalence-tikhonov-ivanov}
For any $\gamma>0$ there exists $\lambda(\gamma)$ such that $G_{\lambda(\gamma)}$ is a minimizer of \cref{eq:ivanov_tn}. Moreover, for any $\lambda>0$ there exists a $\gamma = \gamma(\lambda)>0$ such that $G_\lambda$ is a minimizer of \cref{eq:ivanov_tn}. 
\end{theorem}

\begin{proof}
We first consider the case where, given a $\gamma>0$ we want to relate a solution of the Ivanov regularization problem to that of Tikhonov regularization. We will show that there exists $G^I_\gamma$ and $\lambda(\gamma)$ such that $G_\lambda(\gamma) = G^I_\gamma$. In particular we will show that such equality holds for $G^I_\gamma$ the solution of minimal trace norm in the set of solutions of the Ivanov problem.

Consider again the linear subspace 
\eqals{
    \mathcal{S} = \left\{G\in\hy\otimes\hx ~\middle|~ G = G(\Phi^\dagger\Phi) \right\}.
}
We can restrict the original Ivanov problem to 
\eqals{
    \min_{\substack{\tnorm{G}\leq\gamma\\G\in\mathcal{S}}} ~ \hat\risk(G).
}
Note that the above is still a convex program and attains the same minimum value of the original Ivanov problem in $G^I_\gamma$. 

Moreover, we can assume $\gamma = \tnorm{G^{I}_\gamma}$ without loss of generality. Indeed, if $\gamma>\tnorm{G^{I}_\gamma}$ we still have that $G^{I}_\gamma$ is a minimizer of $\hat\risk(G)$ over the smaller set of operators $\tnorm{G}\leq \gamma' = \tnorm{G^{I}_\gamma}$. 

Now, consider the Lagrangian associated to this constrained problem problem, namely
\eqals{
    L(G,\lambda,\nu) = \hat\risk(G) + \lambda(\tnorm{G} - \gamma) + \nu(G - G(\Phi^\dagger\Phi)).
}
By Slater's constraint qualification \citep[see e.g. Sec. 5 in][]{Boyd:2004}, we have that 
\eqals{
    \max_{\lambda\geq0,\nu}\min_{G\in\hy\otimes\hx} L(G,\lambda,\nu) = \min_{\substack{\tnorm{G}\leq\gamma\\G\in\mathcal{S}}} \hat\risk(G).
}
Denote by $(\lambda(\gamma),G_{\lambda(\gamma)},\nu_\gamma)$ the pair form which the saddle point of $L(G,\lambda,\gamma)$ is attained. Note that since $\gamma$ is a constant
\eqals{
    G_{\lambda(\gamma)} = \argmin_{G\in\hy\otimes\hx} \hat\risk(G) + \lambda(\gamma) (\tnorm{G} - \gamma) + \nu_\gamma(G - G(\Phi^\dagger\Phi)) = \argmin_{G\in\hy\otimes\hx} \hat\risk(G) + \lambda(\gamma)\tnorm{G},
}
where we have made use of the representer theorem from \cref{thm:representer-theorem-trace norm}, which guarantees any minimizer of $\hat\risk(G) + \lambda(\gamma)\tnorm{G}$ to belong to the set $\mathcal{S}$ and this satisfy $G = G(\Phi^\dagger\Phi)$. Therefore, we have 
\eqals{
    \hat\risk(G_{\lambda(\gamma)}) + \lambda(\gamma)\tnorm{G_{\lambda,\gamma}} - \lambda(\gamma)\gamma = \hat\risk(G^{I}_\gamma),
}
recalling that $\gamma = \tnorm{G^{I}_\gamma}$, this implies that 
\eqals{
    \hat\risk(G_{\lambda,\gamma}) + \lambda(\gamma)\tnorm{G_{\lambda,\gamma}} = \hat\risk(G^{I}_{\gamma}) + \lambda(\gamma)\tnorm{G^{I}_{\gamma}}.
}
Since by \cref{prop:trace norm-unique-minimizer} the minimizer of $\hat\risk(G)+\lambda\tnorm{G}$ is unique, it follows that $G_{\lambda,\gamma} = G^{I}_\gamma$ as desired.

The vice-versa is straightforward: let $\lambda>0$ and $G_\lambda$ be the minimizer of the Tikhonov problem. Then, for any $G\in\hy\otimes\hx$
\eqals{
    \hat\risk(G_\lambda)+\lambda\tnorm{G_\lambda} \leq \hat\risk(G)+\lambda\tnorm{G}.
}
If $\tnorm{G}\leq\tnorm{G_\lambda}$, the inequality above implies 
\eqals{
    \hat\risk(G_\lambda)\leq\hat\risk(G),
}
which implies that $G_\lambda$ is a minimizer for the Ivanov problem with $\gamma(\lambda) = \tnorm{G_\lambda}$, namely $G_\lambda = G^I_{\gamma(\lambda)}$ as desired.
\end{proof}

\end{document}